\documentclass[journal]{IEEEtran}

\usepackage{url}
\usepackage{float}
\usepackage{acronym}
\usepackage{comment}
\usepackage{soul}
\usepackage[linesnumbered,ruled,vlined]{algorithm2e}
\usepackage{algpseudocode}
\usepackage{bm}% bold math
\usepackage{amsthm}
\usepackage{amsmath}
\usepackage{amssymb}
\usepackage{cleveref}
\newtheorem{theorem}{Theorem}
\newtheorem{prop}{Proposition}

\newtheorem{corollary}{Corollary}
\newtheorem{definition}{Definition}
\theoremstyle{definition}
\usepackage{xcolor}
\usepackage{cleveref}
\usepackage{multirow}
\usepackage{booktabs}
\usepackage{svg}

\usepackage{subcaption}
\usepackage{graphicx}

\usepackage{thmtools}
\usepackage{thm-restate} 
\def\R{\mathbb{R}}

\def\x{\mathbf{x}}
\def\X{\mathbf{X}}
\def\z{\mathbf{z}}
\def\Z{\mathbf{Z}}

\def\S{\mathbf{S}}
\def\V{\mathbf{V}}
\def\N{\mathbb{N}}
\def\n{\mathbf{n}}
\def\L{\mathcal{L}}

\def\threLi{LI}

\def\c{\mathbf{c}}
\def\C{\mathcal{C}}%cycle basis
\def\Cs{\mathbb{C}}%cycle basis set
\def\a{\mathbf{a}}
\def\H{\mathbf{H}}
\def\S{\mathbf{S}}
\def\J{\text{LNR}}
\def\r{r}
\def\dimx{n}
\def\dimz{m}
\def\time{T}% number of measurement

\def\noise{\boldsymbol{\varepsilon}}
\def\h{\boldsymbol{h}}

\def\eR{\mathbf{R}}

\def\p{\mathbf{p}} %parameter of branches 
\def\abs{\text{abs}}% absolute value
\def\obj{\mathbf{O}} % objective of optimization
\def\G{\mathcal{G}} 
\def\V{\mathcal{V}} 
\def\E{\mathcal{E}} 
\def\siIn{\zeta}
\def\unIn{\chi}

\def\v{v}
\def\e{e}
\def\nul{\mathbf{N}}
\def\A{\mathbf{A}}
\def\cyc{\mathcal{C}} % cycle

\def\thetaAE{\theta_{\text{AE}}}

\def\cc{\boldsymbol{u}}
\def\err{\boldsymbol{\epsilon}} % vector represents the error of estimated \
\def\EXP{\mathbb{E}} %expectation

\def\W{\mathcal{W}}

\acrodef{ai}[AI]{Artificial Intelligence}
\acrodef{bdd}[BDD]{Bad Data Detector} 
\acrodef{cnn}[CNN]{Convolutional Neural Network}
\acrodef{cps}[CPS]{Cyber-Physical System} 
\acrodef{edr}[EDR]{Enforced Dynamic Response} 
\acrodef{fdia}[FDIA]{False Data Injection Attack}
\acrodef{fft}[FFT]{Fast Fourier Transform}
\acrodef{gft}[GFT]{Graph Fourier Transform}
\acrodef{gsp}[GSP]{Graph Signal Processing}
\acrodef{imf}[IMF]{Intrinsic Mode Function}
\acrodef{laa}[LAA]{Load Altering Attack}
\acrodef{lstm}[LSTM]{Long Short-Term Memory}
\acrodef{ml}[DL]{Deep Learning}
\acrodef{mppt}[MPPT]{Maximum Power Point Tracking}
\acrodef{pca}[PCA]{Principal Component Analysis}
\acrodef{pcc}[PCC]{Point of Common Coupling}
\acrodef{pmu}[PMU]{Phasor Measurement Unit}
\acrodef{pscp}[PSCP]{Pseudo-null Space Conserved Data Preprocessing}
\acrodef{rom}[ROM]{Reduced Order Modeling}
\acrodef{std}[STD]{Standard Deviation}
\acrodef{state-estimation}[SE]{State-Estimation}
\acrodef{svd}[SVD]{Singular Value Decomposition}
\acrodef{vmd}[VMD]{Variational Mode Decomposition}
\acrodef{wse}[WSE]{Wavelet Singular Entropy}
\acrodef{ae}[AE]{Autoencoder}
\acrodef{relu}[ReLU]{Rectified Linear Unit}
\acrodef{pr-auc}[PR-AUC]{Area under the precision-recall curve}
\acrodef{csd}[CSD]{Cycle-Space Detector}
\acrodef{mcb}[MCB]{Minimum Cycle Basis}
\acrodef{if}[Isolation Forest]{Isolation Forest}
\acrodef{tls}[TLS]{Total Least Squares}

\title{Cycle-Space Informed Detection of Autoencoded Blind False Data Injection Attacks on Power Systems
}

\IEEEaftertitletext{}
\author{\IEEEauthorblockN{
\textbf{Xin Li}\IEEEauthorrefmark{1},
\textbf{Chenhan Xiao}\IEEEauthorrefmark{2}, 
\textbf{Jonathan Cohen}\IEEEauthorrefmark{1}, 
\textbf{Aviad Elyashar}\IEEEauthorrefmark{3}\IEEEauthorrefmark{4}, 
\textbf{Yang Weng}\IEEEauthorrefmark{2},
\textbf{Rami Puzis}\IEEEauthorrefmark{1}\IEEEauthorrefmark{4}
}

\IEEEauthorblockA{\IEEEauthorrefmark{1}}Faculty of Computer and Information Science, Ben-Gurion-University, Be'er Sheva, Israel \\
\IEEEauthorblockA{\IEEEauthorrefmark{2}School of Electrical, Computer and Energy Engineering, Arizona State University, AZ, USA}\\
\IEEEauthorblockA{\IEEEauthorrefmark{3}Department of Computer Science, Shamoon College of Engineering, Be'er Sheva, Israel}\\
\IEEEauthorblockA{\IEEEauthorrefmark{4}Cyber@BGU, Ben-Gurion University of Negev, Be'er Sheva, Israel}\\
}

\begin{document}

\maketitle
\begin{abstract}
The rapid growth of AI-driven data centers and large-scale energy storage systems is increasing the reliance of power system operation on real-time measurement data and automated decision-making. However, many existing detection methods rely on statistical or data-driven analysis of measurements and can fail when attackers exploit the same data structure to craft stealthy perturbations. To illustrate this limitation, we demonstrate a blind \ac{fdia} in which an autoencoder learns the measurement manifold and generates perturbations aligned with the Jacobian null space, thereby allowing the attack to evade both residual-based bad-data detectors and time-series anomaly detectors.  To mitigate data-driven FDIAs which exploit the null space, we propose a topology-informed \ac{csd} that leverages the Cycle-Space of the network to impose structural constraints that enhance null space estimation. In addition, we prove that by using the \ac{mcb}, the proposed \ac{csd} achieves the optimal generalization error for attack detection. By exploiting topology-derived cycle constraints rather than relying solely on numerical null space estimation, the proposed method does not require precise line parameters and improves the separation between normal and attacked measurements. Simulation results on IEEE 14-, 30-, 57-, and 118-bus systems demonstrate that the proposed method effectively detects data-driven FDIAs under realistic measurement noise.

\end{abstract}


%All manuscripts must be in English. {\em Manuscripts must not exceed 10 pages, single-spaced and double-columned}.  This includes all graphs, tables, figures and references. These guidelines include complete descriptions of the fonts, spacing, and related information for producing your proceedings manuscripts. Please follow the steps outlined below when preparing your final manuscript. {\em Read the following carefully. The quality of the finished product largely depends upon receiving your cooperation and help at this particular stage of the publication process.}

\section{Introduction}
\label{sec:introduction}
Modern power system operation increasingly relies on real-time measurement data, especially with the rapid growth of AI-driven data centers and large-scale energy storage systems \cite{swain2022sensor}. 
However, integrating cyber infrastructure also introduces new vulnerabilities into power systems. %These incidents include the 2015 Ukraine blackout \cite{lightsout}, the 2018 intrusion into the U.S. power grid \cite{cnn2018us}, and recent cyberattacks targeting North American energy entities \cite{cyberattacknews}, 
Recent reports show that cyberattacks targeting energy infrastructure are rapidly increasing, with attacks on U.S. utilities rising by nearly 70\% between 2023 and 2024 \cite{ferc2025summer}. 
At the same time, emerging grid assets, such as utility-scale battery energy storage systems, are becoming new potential targets for cyber intrusions \cite{li2023cybersecuritybess}. 
These risks are particularly significant as more grid data is processed through cloud platforms and third-party services \cite{rasner2021cybersecurity}. Notably, a recent cyberattack on the Polish power grid highlighted the stealth and sophistication of such threats \cite{certpolska2026energy}. %In such environments, ensuring the integrity of measurement data used for system monitoring and state estimation is critical for reliable grid operation. 

One common attack is the false data injection attack (FDIA), which manipulates measurement data used in state estimation. Carefully constructed FDIAs exploit the null space of the measurement Jacobian to bypass traditional bad data detectors (BDDs). 
This allows attackers to alter estimated system states without triggering residual-based alarms~\cite{liu2011false}. 
For example, successful attacks can mislead system operators and lead to severe operational consequences, such as cascading outages \cite{hong2021data}, transmission congestion \cite{paul2022modified}, and economic disruption \cite{rahman2014impact}. %As power systems become increasingly data-driven and automated, the potential impact of corrupted measurement data continues to grow.

Early research on FDIA primarily focused on model-based attacks, in which adversaries assume knowledge of system parameters, such as network topology and line admittances. 
Under this assumption, adversaries can construct attack vectors aligned with the column space of the measurement Jacobian, preserving residual statistics and evading BDD-based detection \cite{liu2011false}. 
Subsequent work investigated techniques for reconstructing system information from measurements to relax the assumption of full model knowledge. 
These methods include Independent Component Analysis (ICA) \cite{esmalifalak2015stealthy}, Principal Component Analysis (PCA) \cite{yu2015blind, anwar2016stealthy}, and low-rank matrix factorization \cite{yang2023false}, which attempt to estimate the Jacobian structure from historical measurements.

More recent studies have explored data-driven approaches for constructing blind FDIAs that do not require explicit knowledge of system parameters. 
This shift is motivated by the limited accessibility of physical system information: grid operators increasingly restrict access to sensitive infrastructure data \cite{anderson2022power}, and in practice, even system operators may lack accurate and up-to-date system models due to evolving infrastructure, third-party assets, and infrequent model updates \cite{basit2020limitations, rasner2021cybersecurity}. 
To address these limitations, blind FDIA methods attempt to infer system structure directly from measurements. 
Examples include eigenvalue-based matrix reconstruction \cite{yang2022blind, lakshminarayana2020data} and deep learning approaches, such as autoencoders \cite{musleh2022attack} and generative adversarial networks (GANs) \cite{shahriar2021iattackgen}. 
While these techniques demonstrate that stealthy attacks can be constructed solely from measurement data, they also expose a key limitation of many existing detection strategies. 
When both attack generation and detection rely primarily on numerical or data-driven analysis of the same measurements, the resulting symmetry makes it difficult to distinguish malicious perturbations from normal system variability reliably.

To expose the limitations of existing detection methods, we study a blind FDIA generated by a tailored autoencoder (AE) that learns the measurement manifold from historical data and produces perturbations aligned with the Jacobian null space. Such perturbations remain close to the learned measurement manifold and can bypass both residual-based bad data detectors (BDDs) and time-series anomaly detectors. These observations highlight the difficulty of detecting attacks using purely statistical or data-driven analysis of measurements. 

To address this challenge, we propose a topology-informed detection framework, the \acf{csd}, that leverages the Cycle-Space structure of the power network to constrain null-space estimation and detect data-driven FDIAs.
In DC power systems, the Cycle-Space, defined as the kernel of the incidence matrix, captures the fundamental loops whose structure is closely related to the null space of the measurement Jacobian. 
By incorporating these topology-derived constraints, the proposed method introduces structural information that cannot be inferred solely from measurement data.
The method constructs a cycle basis for the network and uses cycle-consistent indicators to guide null-space estimation, enabling reliable detection even when precise line parameters are unavailable.

To the best of our knowledge, this work is the first to bridge graph-theoretic Cycle-Space theory with power system state estimation to provide a topology-informed defense against blind \ac{fdia}s.

As a supporting analytical result, we derive a first-order finite-sample expression for the expected generalization error of the null-space estimator. %under isotropic Gaussian design.
Building on the classical first-order subspace perturbation framework of \cite{swewart1977perturbation,li2002performance}, we derive a compact, closed-form expression for the generalization error in terms of $\text{rank(}\H)$ and the sample size $\time_o$.
While classical subspace-estimation quantities, such as \ac{tls} covariance \cite{fierro1993total}, DOA parameter MSE \cite{li2002performance}, and singular-vector overlaps \cite{benaych2011eigenvalues}, are well established, the train/test functional we target and its closed form have not, to our knowledge, been derived previously. While presented here in the context of power system cybersecurity, the result is mathematically general and offers a new analytical tool for the \ac{tls} community.
This analysis leads to Corollary 1, which shows that the \ac{mcb} minimizes the generalization error among all cycle-basis null-space estimators.

Simulation studies on IEEE 14-, 30-, 57-, and 118-bus transmission systems demonstrate that the proposed detection framework effectively identifies data-driven FDIAs under realistic measurement noise.
The results show that incorporating topology-derived cycle constraints improves the separation between normal and attacked measurements compared with purely data-driven anomaly detection approaches.

The remainder of the paper is organized as follows. Section~\ref{sec:preliminaries} reviews DC state estimation and BDD fundamentals. Section~\ref{sec:AE-method} describes the autoencoder-based blind FDIA scheme. Section~\ref{sec:cycle_informed_method} presents the proposed Cycle-Space-informed detection framework. Section~\ref{sec:evaluation} reports the numerical evaluation results, and Section~\ref{sec:conclusion} concludes the paper.

\section{Preliminaries}
\label{sec:preliminaries}

We adopt the DC state estimation formulation. The DC power grid is modeled as a graph $\G=(\V,\E)$ where $|\V|=n$ denotes the set of buses and $|\E|=m$ the set of branches. Let $\x = (x_1, \cdots, x_{\dimx}) \in \mathbb{R}^{\dimx}$ denote the system state vector (e.g., voltage phase angles at buses), and $\z = (z_1, \cdots, z_{\dimz}) \in \mathbb{R}^{\dimz}$ denote the system measurement vector (e.g., active power flows on branches). The linearized DC power flow model relates them through the system Jacobian matrix $\H \in \mathbb{R}^{\dimz \times \dimx}$, which encodes network topology and measurement configuration:
\(
    \z = \H \x + \noise, 
\)
where $\noise \in \mathbb{R}^{\dimz}$ represents measurement noise due to sensor error or communication disturbances. The noise is typically modeled as zero-mean Gaussian, i.e., $\noise \sim \mathcal{N}(\boldsymbol{0}, \eR)$, where $\eR = \mathrm{diag}(\sigma_1^2, \cdots, \sigma_{\dimz}^2)$ is the noise covariance matrix. 
Based on such a relationship above, power system state estimation estimate the state $\hat{\x}$ from measurements $\z$ by solving the weighted least-squares problem:
\vspace{-0.5em}
\begin{align}
    \hat{\x} &= \arg\min_{\x} \sum\nolimits_{i=1}^{\dimz} \frac{(z_i - \H_i \x)^2}{\sigma_i^2} \label{eq:SE}\\
    &=(\H^T \eR^{-1} \H)^{-1} \H^T \eR^{-1} \z, \label{eq:SE-closed}
\end{align}
where $\H_i$ is the $i$-th row of $\H$. The estimated state $\hat{\x}$ help decide critical monitoring and control decisions, making it a high-value target for adversarial attacks. 
To verify the measurement data $\z$, a widely used approach is Bad Data Detection (BDD) \cite{abur2004power,wang2020detection}. BDD evaluates the consistency between observed measurements and the estimated state by examining the squared residual error:
\(
    \|\z - \hat{\z}\|_2^2 = \|\z - \H \hat{\x}\|_2^2 = \|\S \z\|_2^2,
\)
where $\S = \mathbf{I} - \H(\H^T \eR^{-1} \H)^{-1} \H^T \eR^{-1}$ is the residual sensitivity matrix~\cite{abur2004power}. 
Under the assumption of Gaussian noise, the normalized residual
\begin{equation}
    \J(\z) := \sum\nolimits_{i=1}^{\dimz} \frac{(\S_i \z)^2}{\sigma_i^2} = \sum\nolimits_{i=1}^{\dimz} \frac{(\S_i \noise)^2}{\sigma_i^2} \sim \chi^2_{\dimz - \dimx}, 
    \label{eq:Chi-squared}
\vspace{-0.5em}
\end{equation}
follows a $\chi^2$ distribution with $\dimz - \dimx$ degrees of freedom. This property enables a statistical detection rule: for a given significance level $\alpha$ (e.g., 0.05), the residual statistic $\J(\z)$ is compared with the threshold $\tau = \chi^2_{\dimz - \dimx,\, 1-\alpha}$. If $\J(\z) \geq \tau$, the measurement set is flagged as containing bad data; otherwise, it is deemed consistent with the assumed noise model.

\section{Blind FDIA via Measurement Manifold Learning}
\label{sec:AE-method}

Traditional FDIA schemes assume knowledge of system information, in particular the Jacobian matrix $\H$. Under this assumption, an adversary can modify the measurement vector as $\z' = \z + \H\cc$~\cite{liu2011false}, where \( \cc \in \mathbb{R}^{\dimx} \) is an arbitrarily chosen direction in the state space. Because the perturbation $\H\cc$ lies in the column space of \( \H \), the attacked measurement preserves the residual error, i.e., \( \| \S (\z + \H\cc) \| = \| \S \z \| \), allowing the attack to remain undetected by the \ac{bdd}. However, this model-based approach has two important limitations. First, access to the exact Jacobian matrix is often unrealistic in practice because such information is sensitive and typically protected by utility operators. Second, although these attacks evade \ac{bdd}, they may still introduce abnormal measurement patterns that deviate from historical data, making them vulnerable to modern time-series anomaly detectors \cite{zu2024self}.

To address these limitations, we construct a blind FDIA that does not require explicit knowledge of the Jacobian matrix. The key observation is that under the DC model \( \z = \H\x \), the measurement manifold corresponds to the column space of \( \H \). Therefore, if an attacker can learn this manifold from historical measurements, stealth attack directions can be recovered without explicit knowledge of the system model. Based on this insight, we employ a tailored autoencoder (AE) to learn the measurement manifold directly from historical system data.

To ensure that the \ac{ae} faithfully captures the structure of the measurement manifold, we design its latent space dimensionality to match the number of system states \( \dimx \). While the exact Jacobian matrix \( \H \) is typically unavailable to attackers, the value of \( \dimx \) can often be inferred from system metadata or estimated using model selection and dimensionality reduction techniques. This architectural choice forces the AE to encode only the minimal set of variables required to reconstruct valid measurements, aligning the learned representation with the physical notion of system states in power systems. 
% In Section~\ref{sec:AE-latentdimension}, we numerically validate that the proposed blind \ac{fdia} scheme performs best when the latent space dimensionality coincides with the state space.

Moreover, the \ac{ae} architecture naturally performs a denoising operation. Since the observed measurements follow the noisy model \( \z = \H \x + \noise \), where \( \noise \) denotes Gaussian sensor noise, the AE is incentivized to recover the signal component \( \H \x \), while suppressing the noise \( \noise \). Similar to linear \acp{ae}, which are mathematically equivalent to \ac{pca} \cite{baldi1989neural}, the latent representation captures directions of maximum variance in the data and filters out noise components. This property is useful for constructing stealthy perturbations that remain close to the learned measurement manifold.

\begin{figure}[h]
    \centering
    \vskip -0.1in
    \includegraphics[width=0.65\linewidth]{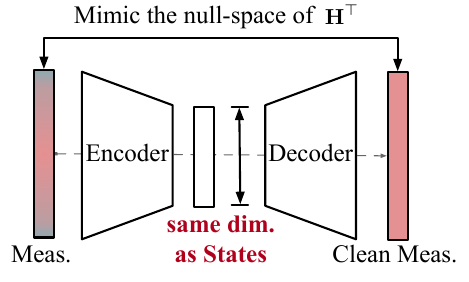}
    \vskip -0.1in
    \caption{Customized autoencoder architecture designed to mimic the state estimation process and implicitly learn the null space of the system Jacobian matrix \( \H^T \).}
    \label{fig:AE}
    \vskip -0.1in
\end{figure}

The proposed \ac{ae} architecture is illustrated in Fig.~\ref{fig:AE}. It consists of an encoder that maps input measurements \( \z \in \mathbb{R}^{\dimz} \) to a lower-dimensional latent representation, and a decoder that reconstructs the measurement from the latent space \cite{goodfellow2016deep}. The \ac{ae} model \( \rm{AE}(\cdot; \thetaAE) \), where \( \thetaAE \) denotes the network parameters, is trained to minimize the reconstruction loss
\begin{equation}\label{eq:AEmodel}
    \min_{\thetaAE} \mathbb{E}_{\z}\left\| \z - \rm{AE}\left(\z; \thetaAE\right) \right\|^2.
\end{equation}

After the \ac{ae} model in~\cref{eq:AEmodel} converges, the attacker can generate an FDIA as follows (see Fig.~\ref{fig:reconstruct}). Given a real-time measurement \( \z_t \), the \ac{ae} reconstructs it as \( \hat{\z}_t = \rm{AE}(\z_t) \). The reconstruction residual
\(
\r_t := \z_t - \hat{\z}_t
\)
captures the component of \( \z_t \) that lies outside the learned manifold and therefore approximately lies in $\text{Col}(\H)^{\perp}$, where $\text{Col}(\H):=\{\H\x|\x\in\R^{\dimx}\}$ denotes the column space of $\H$. This orthogonal complement of $\text{Col}(\H)$ serves as a proxy for the system’s null space. The attacker then perturbs the real-time measurement along the residual direction
\begin{equation}\label{eq:AE-residual-FDIA}
\z'_t = \z_t + \kappa \cdot \r_t,
\end{equation}
where \( \kappa > 0 \) controls the attack magnitude. Since $\hat{\z}_t$ lies on the learned manifold, the perturbed measurement \( \z'_t \) remains close to historical measurement patterns, reducing the risk of detection by time-series anomaly detectors. Moreover, because $\r_t \perp \text{Col}(\H)$, it approximately lies in the null space of the sensitivity matrix \( \S \), i.e., \( \S \r_t \approx 0 \), which implies that
\(
\S \z'_t = \S \z_t + \kappa \S \r_t \approx \S \z_t.
\)
Thus, the attack preserves the residual used in BDD and successfully bypasses BDD-based detection.

\begin{figure}[h]
    \centering
    \vskip -0.1in
    \includegraphics[width=0.9\linewidth]{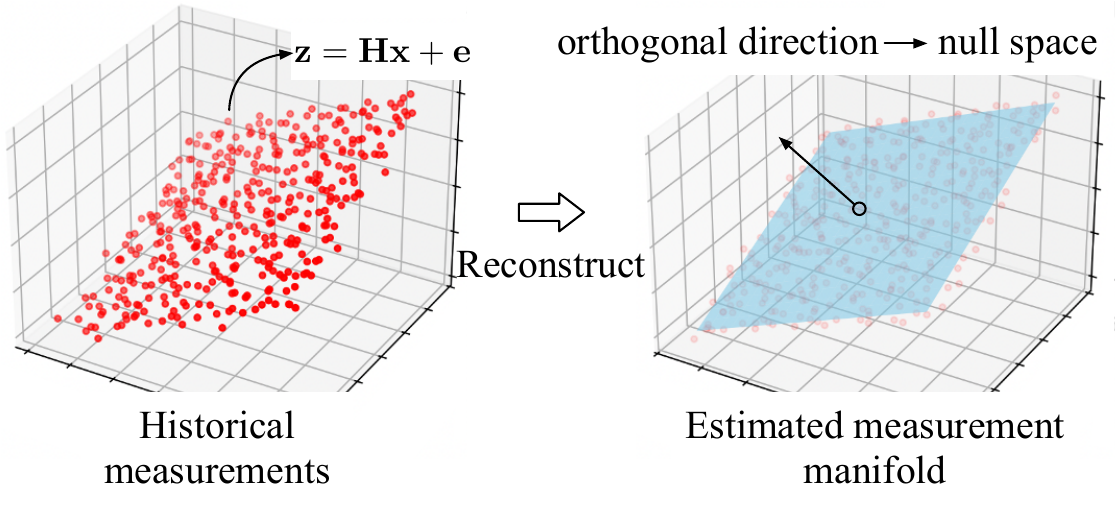}
    \vskip -0.1in
    \caption{Reconstruction of measurement manifold from historical observed measurements.}
    \label{fig:reconstruct}
    \vskip -0.1in
\end{figure}

To further reduce the detectability of the attack and ensure statistical indistinguishability from normal measurement noise, we add a Gaussian perturbation \( \boldsymbol{\eta}_t \sim \mathcal{N}(\boldsymbol{0}, \gamma \eR) \), where \( \gamma \in (0,1) \) controls the variance. The final injected measurement becomes
\(
\z'_t = \z_t + \kappa \cdot (\r_t + \boldsymbol{\eta}_t).
\)
This additional perturbation captures sensor-level variability and helps the attack remain hidden even under fine-grained time-series anomaly detection methods. Unlike traditional FDIA constructions that rely on explicit system matrices, the proposed scheme infers stealth directions directly from historical data without access to \( \H \).

\subsection{Pseudo-Null Space FDIA extended to AC Systems}
\label{sec:attack-AC}
While the above FDIA is formulated under the linear DC model, the underlying concept generalizes to the nonlinear AC setting, where the measurement model becomes $\z=\h(\x)+\noise$. In this case, due to the slack bus constraint and physical conservation laws, the image of $\h(\cdot)$ forms a lower-dimensional nonlinear manifold within the measurement space. Compared to DC system based FDIA, constructing stealthy attacks in this setting requires identifying perturbations $\c_p$ that satisfy $\a = \h(\x+\c_p)- \h(\x)\approx 0$, thereby preserving the measurement consistency necessary to evade \ac{ml}-based anomaly detection. One strategy for achieving this condition is to formulate a nonlinear optimization problem that directly minimizes the measurement discrepancy: $\min_{\c_p}\|\h(\x+\c_p)- \h(\x)\|^2$. However, solving such nonlinear optimization problem can be computationally intensive and impractical for large-scale systems or real-time applications.

To address this challenge, we adopt an alternative strategy based on a first-order Taylor approximation of $\h(\cdot)$. Linearizing the measurement function around $\x$ yields $\h(\x+\c_p)\approx \h(\x)+\boldsymbol{J}(\x)\c_p$, where $\boldsymbol{J}(\x)$ denotes the Jacobian matrix of $\h$ evaluated at $\x$. Stealthiness can then be enforced by selecting $\c_p\in \text{null}(\boldsymbol{J}(\x))$, ensuring that the perturbation lies in the local linear kernel of the measurement function and thus approximately preserves measurement consistency. This Jacobian-based method is computationally efficient, scalable, and effective when the perturbations are sufficiently small to remain within the linear approximation region. Notably, this procedure closely parallels the approach used in the DC case. The key difference lies in the fact that, under the AC model, the Jacobian $\boldsymbol{J}$ must be re-evaluated at each operating point for FDIA construction, whereas the DC formulation involves a constant Jacobian $\H$. In practice, the update frequency of $\boldsymbol{J}$ is inherently determined by the sampling rate of the system measurements.

\section{Cycle-Space-Based Detection of Data-Driven Blind \ac{fdia}}
\label{sec:cycle_informed_method}

\begin{figure*}
\centering
\vskip -0.10in
\includegraphics[width=0.7\linewidth]{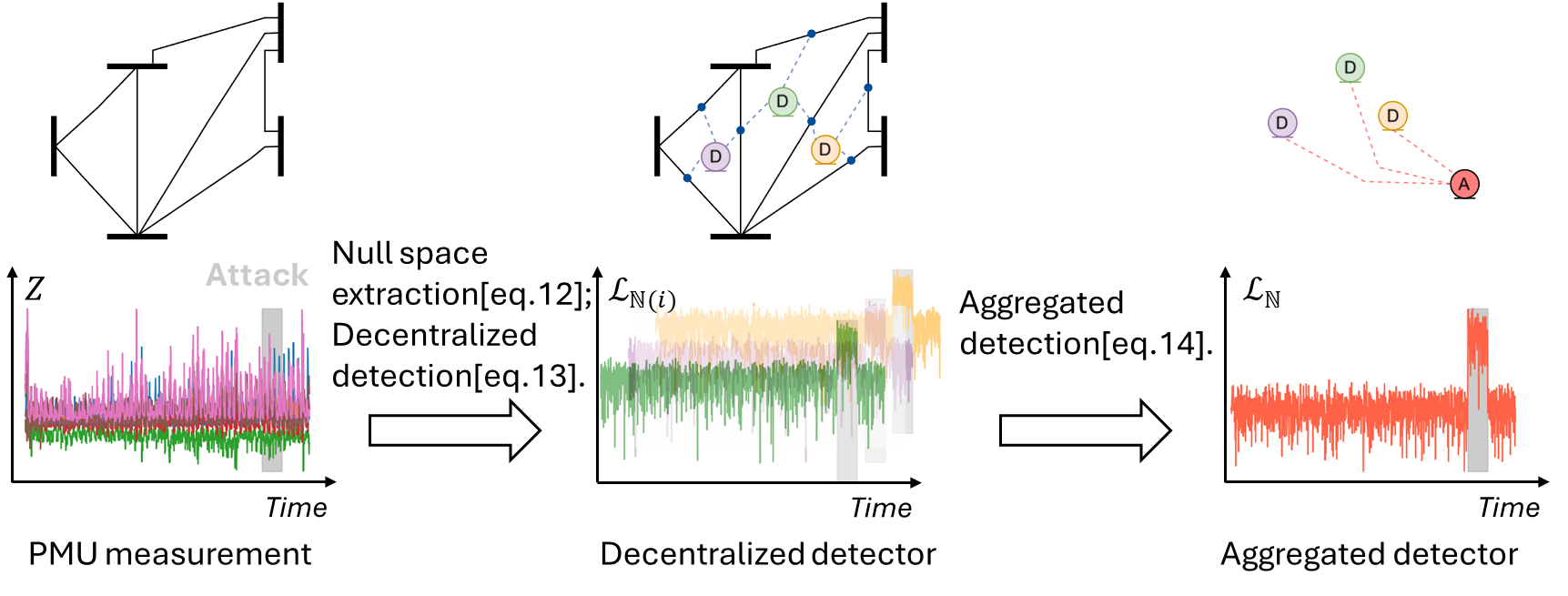}
\caption{Overview of the proposed Cycle-Space informed FDIA detection procedure.}
\label{fig:pipeline}
\vskip -0.10in
\end{figure*}

The blind FDIA in Section~\ref{sec:AE-method} approximates the null space of the Jacobian matrix $\H^T$ using historical measurements. While this data-driven construction enables the attack to bypass both BDDs and time-series anomaly detectors, its null space approximation is inherently imperfect. As a result, the injected perturbation does not lie exactly in the true null space, leaving a detectable footprint.
This observation motivates the proposed detection method. Existing null space-based detection schemes typically rely on purely numerical tools such as Singular Value Decomposition (SVD) to estimate the null space of $\H^T$. However, such methods use only measurement data and do not exploit the physical structure of the power network. Consequently, they provide limited advantage to the defender when the attacker also relies on data-driven null space approximation. In contrast, system operators typically have access to the network topology. We leverage this structural information to construct a more physically meaningful approximation of the null space and use it for attack detection.

In DC power systems, the non-trivial null space of the Jacobian matrix arises from cycles in the network graph. In particular, radial networks do not contain closed loops and therefore have a trivial Cycle-Space. This motivates the use of Cycle-Space theory for null space reconstruction. The resulting topology-informed approach offers two advantages. First, unlike numerical decomposition methods such as SVD, it does not rely on precise values of internal system parameters such as line admittances. Second, it naturally supports decentralized detection because each cycle can serve as an independent local consistency check. Fig.~\ref{fig:pipeline} provides an overview of the proposed detection framework, which consists of two components: (1) Cycle-Space informed null space extraction, and (2) decentralized attack detection.

\subsection{Cycle-Space Informed Null Space Extraction}

We first introduce the graph-theoretic objects used to construct the topology-informed null space.

\begin{definition}[Cycle]\label{def:cycle}
Let $\G = (\V, \E)$ be a connected graph with vertices (buses) $v_i \in \V$ and edges (branches) $e_i \in \E$. A \emph{cycle} is a closed walk of the form $[\v_1, \e_1, \v_2, \ldots, \e_{l-1}, \v_l]$ such that $v_1 = v_l$ and all other vertices are distinct. Equivalently, we represent a cycle by the set of edges as $\c=\{\e_1, \e_2, \ldots, \e_{l-1}\}$.
\end{definition}

The dimension of the null space of a graph's incidence matrix corresponds to the number of independent cycles in the graph. We now extend this concept to the algebraic representation of all independent cycles.

\begin{definition}[Cycle basis]\label{def:cycle_basis}
A cycle basis $\cyc=\{\c_1,\c_2,\ldots,\c_{m-n+1}\}$ is a minimal set of simple cycles that can be combined to generate every cycle in the graph. Each cycle $\c \in \cyc$ is associated with a signed indicator vector $\siIn_{\c}\in \mathbb{R}^m$, which encodes the orientation of an edge in the cycle relative to its orientation in the graph:
\begin{equation}
\label{eq:weighted_signed_indicator}
\siIn_{\c}(\e) =
\begin{cases} 
    1, & \text{if $\e\in\c$ and its orientations in $\c$} \\ &\text{and $\G$ are the same} \\
   -1, & \text{if $\e\in\c$ and its orientations in $\c$} \\ &\text{and $\G$ are different} \\
   0, & \text{if $\e\notin \c$ } 
\end{cases}
\end{equation}
where $e\in\E$ denotes an edge in the graph $\G$.
\end{definition}

The cycle basis is generally non-unique. The cycle basis set $\Cs=\{\cyc^{(1)}, \cyc^{(2)},\ldots\}$ denotes the collection of all such cycle bases.

\begin{definition}[Cycle-Space]
\label{def:weighted_cycle_space}
The Cycle-Space of the graph $\G$ under cycle basis $\C$ is defined as $\nul_c=\siIn\odot\p =(\n^{(1)}_c,\n^{(2)}_c,\ldots,\n^{(\dimz-\dimx+1)}_c)\in\R^{\dimz\times|\cyc|}$,
where $\p$ is the parameter of the branch (i.e., the reactance).

Since the Cycle-Space coincides with the null space of $\H^T$, i.e., $\nul_c=Null(\H^T)$, each direction of $\nul_c$ admits:
\begin{equation}
    \bar\Z^T\n^{(i)}_c=0
\end{equation}
where $\bar\Z$ is clean measurement data without noise and attack.
\end{definition}

The orthogonal complement is the cut space:

\begin{equation}\nul_c^{\perp}= Range(\H).\end{equation}

Our goal is to extract the null space of the matrix $\H^T$ from measurement $\Z$. The null space of the matrix $\H^T$ is defined as $\nul=(\n^{(1)},\n^{(2)},\ldots,\n^{(\dimz-\dimx+1)})$. Each $\n^{(i)}\in\R^{\dimz}$ is a direction of the null space. The null space reconstruction problem is usually formulated as the unconstrained optimization
\begin{equation}
\label{eq:cycle_optimization1}
\begin{aligned}
\obj= \min_{\nul(i)} \ & \sum\nolimits_i ||\Z_{train}^T\cdot \nul(i)||^2
\\
\text{s.t.} \ & \quad ||\n^{(i)}||^2=1.
\end{aligned}
\end{equation}

Based on Definition~\ref{def:weighted_cycle_space}, we reformulate the problem by constraining the null space using the cycle structure:
\begin{equation}
\label{eq:cycle_optimization2}
\begin{aligned}
\obj= \min_{\nul(i)} \ & \sum\nolimits_i ||\Z^T\cdot \nul(i)||^2
\\
\text{s.t.} \ & \quad ||\nul(i)||^2=1,
\quad \nul= \unIn_{\c_i} \odot \nul,
\end{aligned}
\end{equation}
where $\Z=\{\z_1,\ldots,\z_{\time_o}\}\in\R^{\dimz\times\time_o}$ contains the noisy measurements, $\time_o\geq\dimz$ is the number of measurements considered for optimization, $\unIn_{\c}=\abs(\siIn_\c)$ is the unsigned indicator, and $\unIn_{\c}\odot \nul$ regulates the dimensions of the null space spanned by each cycle, where $\odot$ denotes element-wise multiplication. We define the selection $\Z_{\c_i}$, which selects the branch data from $\Z$ that belongs to cycle $\c_i$, such that the dimension of the extracted results is $\R^{|\c_i|\times\time_o}$.
Thus, solving this optimization problem is equivalent to optimizing each cycle separately:
\begin{equation}
\label{eq:cycle_optimization_sub}
\begin{aligned}
\obj_{i} = \min_{\nul_{\c_i}(i)} \ & \sum\nolimits_i \left\| \Z_{\c_i}^T \cdot \nul_{\c_i}(i) \right\|^2 \\
\text{s.t.} \ & \left\| \nul_{\c_i}(i) \right\|^2 = 1,
\end{aligned}
\end{equation}
where $\nul_{\c_i}(i)$ is the optimization entity with dimension $\R^{|\c_i|\times1}$, and $\nul_{\c_i}(i)$ is the $i$-th basis selecting branches that belong to $\c_i$. For brevity, we denote $\nul_{\c_i}(i)$ as $\nul_{\c_i}$.

The solution of \cref{eq:cycle_optimization_sub} is composed of the null space of each selection. Because the rank of $\Z_{\c_i}$ is $|\c_i|-1$, there is only one null space vector. The solution is the eigenvector of $\Z_{\c_i}^T\Z_{\c_i}$ that corresponds to the smallest eigenvalue:
\begin{equation}
\label{eq:cycle_optimization_sub_solution}
\hat{\nul}_{\c_i}(i)= \text{eigvec}_{min}(\Z_{\c_i}\Z_{\c_i}^T)
\end{equation}
After obtaining all $\hat{\nul}_{\c_i}(i)$, we recover the full estimated null space, $\hat{\nul}$, which can be directly used for \ac{fdia} detection. 
Another advantage of the Cycle-Space null space reconstruction is the minimum number of measurements required, namely the number of links of the largest cycle, $\max|\c|$, in contrast to other subspace learning methods, which require at least $\text{rank}(\H)$ measurements with $\max|\c|\leq\text{rank}(\H)$.

By mapping the null space of the Jacobian matrix $H^T$ directly to the Cycle-Space of the network graph, we introduce a novel structural constraint. This represents the first application of the Cycle-Space formalism, grounded in the network topology, to isolate malicious data injections in power grid systems.
Based on~\cref{eq:cycle_optimization_sub,eq:cycle_optimization_sub_solution}, we next construct a bi-level decentralized detection framework.

\subsection{Decentralized Attack Detector}

The Cycle-Space construction naturally supports decentralized detection. In \cref{eq:cycle_optimization_sub}, $\Z_{\c_i}$ contains the measurements of the branches in cycle $\c_i$, while $\nul_{\c_i}$ is the null space basis associated with that cycle. Since the null space originates from the Cycle-Space, attack detection can also be carried out independently on each cycle. Therefore, each cycle can be assigned to a local detector that collects the corresponding branch measurements and performs detection in parallel.

Figure~\ref{fig: 2-layer-structure} illustrates this architecture for the IEEE 14-bus system after removing leaf branches that do not belong to any cycle. The D-detectors (blue nodes) represent local cycle-based detectors, while the blue dashed lines indicate the information flow from branch measurements to the local detectors. The A-detector (red node) aggregates the local detector outputs, as indicated by the red dashed lines.

Since the suspicious components lie in the null space, we project the measurements onto the estimated null space in order to extract the null space component. Let $\bar{\Z}=\{\bar{\z}_1,\ldots,\bar{\z}_{\time_u}\}$ be the measurements under consideration, where $\time_u$ is the number of measurements. The null space component extracted by the detector associated with cycle $\c_i$ is
\begin{equation}
\label{eq:reconstruction_vector_local}
\N_t(i)=(\hat{\nul}_{\c_i}\hat{\nul}_{\c_i}^+\bar{\Z}_{\c_i}(t))^T= \bar{\Z}_{\c_i}^T\hat{\nul}_{\c_i}^T\hat{\nul}_{\c_i}
% \vspace{-1em}
\end{equation}
where $\hat{\nul}_{\c_i}^+ \in \R^{1\times|\c_i|}$ is the pseudo inverse of matrix $\hat{\nul}_{\c_i}$, and $\hat{\nul}_{\c_i}\hat{\nul}_{\c_i}^+$ projects onto the null space component.

Next, we define the magnitude of $\N(i)$ as $\L_{\N(i)}=(\|\N_{1}(i)\|_2, \ldots, \|\N_{\time_u}(i)\|_2)$. We observe that $\L_{\N(i)}$ is smaller under normal operating conditions and therefore provides good separation between normal and attacked measurements. We use $\L_{\N(i)}$ as the reconstruction error of the $i^{\text{th}}$ cycle.

\begin{figure}[h]
    \centering
    \includegraphics[width=0.6\linewidth]{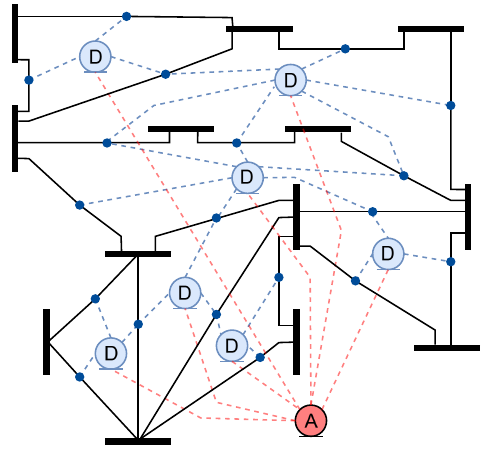}
    \caption{The structure of the decentralized detectors (D-nodes, blue) and aggregated detector (A-node, red) for the IEEE 14-bus system. }
    \label{fig: 2-layer-structure}
\end{figure}

In this two-layered protection system, each decentralized detector can independently detect the attack while also sending a reconstruction error to the aggregated detector (see Fig.~\ref{fig: 2-layer-structure}). We compute the aggregated reconstruction error as 
\begin{equation}
    \L_{\N}=\sqrt{\sum\nolimits_i^{|\cyc|} \L_{\N(i)}^2}.
\end{equation}

To detect the attack interval, if the data $\bar{\Z}$ contains attacked measurements, the reconstruction error should contain two separate classes corresponding to normal and attacked samples. We therefore adopt the minimum cross-entropy criterion \cite{li1993minimum,li1998iterative} to iteratively determine the threshold $\tau_{\N(i)}$ that separates these two classes, by setting $\tau_{\N(i)} = \threLi(\L_{\N(i)})$ \cite{li1993minimum,li1998iterative}. If $\L_{\N_t(i)} \geq \tau_{\N(i)}$, then the detector raises an alarm. Note that $\threLi()$ computes the threshold in linear time and is therefore suitable for real-time implementation.

The experimental detection results in Section~\ref{sec:evaluation} empirically validate that the null space of the Jacobian matrix aligns with the Cycle-Space of the network. Furthermore, in Section~\ref{sec:exp-compare-detectors}, we compare the performance of the proposed Cycle-Space-informed detection framework against state-of-the-art \ac{ml}-based anomaly detectors under the proposed blind \ac{fdia}. In addition, in Section~\ref{sec:exp-agg_vs_single} we evaluate the contribution of each individual decentralized cycle-based detector (D-detector), assessing its standalone detection capability relative to the aggregated detection outcome.

{
\subsection{Optimal Generalization Error}

To analyze the optimality of the proposed \ac{csd}, we derive a finite-sample expression for the expected generalization error using first-order subspace perturbation analysis~\cite{swewart1977perturbation,li2002performance}.

Let:
\begin{itemize}
    \item  $\time_\star$ denote the number of samples in all test sets.
    \item $x_t\in \mathbb{R}^{\dimx}$ be independent and identically distributed (i.i.d.) random state vectors drawn from an isotropic standard normal distribution: $x_t \sim \mathcal{N}(0, I_n)$. And $X_\star=[ x_1,x_2, \dots ]\in \mathbb{R}^{n \times \time_\star}$.
    \item $\Z_\star=\H \X_\star+ \E_\star \in \mathbb{R}^{\dimz \times \time_\star}$ be the noisy test set measurements, produced from independent state and noise realizations $\X_\star$ and $\E_\star$ respectively. $\bar{\Z}_\star=\H \X_\star$ is the noise-free test set measurement.
    \item $\err=\hat{\n}-{\n}$ be the training null space estimation error.
    \item  $ E_{\text{gen}}$ be the scalar values indicating the Expected Generalization errors.
    \item  $\EXP_\Z$ denote an expectation over the training measurements $\Z$. 
    \item  $\EXP_{{\z} \sim \Z_\star}$ denote an expectation over a random sample drawn from the test measurements $\bar{\Z}_\star$.
\end{itemize}

The Expected Generalization Error~\cite[p.34]{james2013introduction}, denoted as $ E_{\text{gen}}$, represents the expected residual error of the estimated null vector $\hat{\n}$ when applied to the underlying distribution of the noisy measurements. $ E_{\text{gen}}$ is defined as the average error that the estimator $\hat{\n}$ incurs over all possible system measurement realizations. It can be regarded as the limit of the $\|\z^T \hat{\n}\|^2$ error as the size of the test set $\Z_\star$ approaches infinity ($\time_\star \to \infty$), ensuring the empirical average converges to the true expectation. Consequently, we adapt the definition of $ E_{\text{gen}}$ to the null space estimation problem as follows:
\begin{equation}
    \label{eq:gen_error}
    \begin{aligned}
     E_{\text{gen}}  =\EXP_{\Z} \left[ \EXP_{\z \sim \Z_\star} [ \|\z^T \hat{\n}\|^2 ] \right]= \lim_{\time_\star \to \infty}  \frac{\EXP_{\Z} \left[  \| \Z_\star^T \hat{\n} \|^2 \right]}{\time_\star} .
    \end{aligned}
\end{equation}

Plugging $\hat{\n}= \n+\err$ into \cref{eq:gen_error}, we can derive the simplified equation (see \Cref{sec:gen_error_simle} for a detailed derivation):

\begin{equation}
\label{eq:final_generalization}
E_{\text{gen}}=\sigma^2+\lim_{\time_\star\rightarrow\infty}\frac{\EXP_\Z[\|\Z_\star^T\err\|^2+2\n^T\E_\star \E_\star^T\err]}{\time_\star}
 \end{equation}

The presence of $\err$ in \cref{eq:final_generalization} underscores the need to investigate its behavior in depth. Hence, we next derive $\err$ analytically in closed form.

The estimation error $\err$ is a nonlinear function of the noise realization $\E$, and its exact distribution does not admit a tractable closed form.
%In the high-dimensional regime $\dimx, \time_o \to \infty$ with $\dimx/T_0\in (0,1)$, the squared overlap $|\langle \hat \n, \n\rangle|^2$ exhibits the Baik–Ben Arous–Péché phase transition (BBP) \cite{baik2005phase, benaych2011eigenvalues}, showing that exact null-vector recovery is not achievable in finite samples.
In the high-dimensional regime $\dimx, T_0 \to \infty$ with $\dimx/T_0 \in (0,1)$, the squared overlap $|\langle \hat{\mathbf{n}}, \mathbf{n} \rangle|^2$ undergoes a phase transition \cite{baik2005phase, benaych2011eigenvalues}, suggesting that exact null-vector recovery is generally not attainable with finite samples.
We therefore use a first-order Lagrangian expansion, valid for small noise variance, and empirically verify its accuracy.
In this setting, higher-order terms $\E^T\err$ and $\err^T\err$ are neglected for simplicity.

\begin{restatable}{prop}{Mprop}[First Order Lagrangian Solution]
\label{eq:prop_delta_solution}
The first order solution for the vector $\err$ which minimizes $\|\Z^T\hat{\n} \|^2$ is given by:
\begin{equation}
\begin{aligned}
    \err =& -(\bar\Z^T)^{+} \E^T \n
\end{aligned}
\label{eq:delta_solution}
\end{equation}
\end{restatable}
See the proof of Proposition~\ref{eq:prop_delta_solution} in \Cref{sec:lagrangian_proof}.  
Since $\EXP[\E^T]= 0_{\time_o \times m}$, the expectation of $\err$ is also ${0}$.
Plugging \cref{eq:delta_solution} into \cref{eq:final_generalization}, we aim to cancel out $\Z_\star,\err,\n$ and $\E_\star$.

Next, we adapt the classical subspace-perturbation machinery of \cite{swewart1977perturbation,li2002performance} to the expected generalization error $\EXP_{\Z} \left[ \EXP_{\z \sim \Z_\star} [ \|\z^T \hat{\n}\|^2 ] \right]$.
While related quantities \ac{tls} coefficient covariance \cite{fierro1993total}, parameter MSE for subspace-based DOA estimators \cite{li2002performance}, and asymptotic singular vector overlaps \cite{benaych2011eigenvalues} are well established, the specific closed-form expression for the expected generalization error of the null space estimator, to our knowledge, has not been derived in prior subspace-estimation work.
Using the orthogonal projection property~\cite{swewart1977perturbation}, namely that $\H\H^T (\H\H^T)^+ = P_{\mathrm{Col}(\H)}$, we establish the relationship between generalization error and $\text{rank}(\H)$.

\begin{restatable}{theorem}{Mtheory}[First Order Finite Sample Expected Generalization Error]
The first-order finite sample expected generalization error is given by:
\begin{equation}
\boxed{
 E_{\text{gen}}=\sigma^2(1+
 \frac{\mathrm{rank}(\H)}{\time_o-n})
}
\label{eq:pop_error}
\end{equation}
\label{eq:theo_pop_error}
\end{restatable}
See the proof of Theorem~\ref{eq:theo_pop_error} in \Cref{sec:gen_error_proof}.
An experimental validation of \cref{eq:pop_error} can be found in \Cref{sec:gen_error_vali}.

To adapt \cref{eq:pop_error} into the Cycle-Space method, we proceed as follows. Each cycle $c$ of the power-system network yields a submatrix $\H_c^T$ whose null space dimension satisfies:
\begin{equation}
\mathrm{nullity}(\H_c^T) = 1.
\end{equation}
Let $\n_c$ denote the null space estimation of cycle $c$. Applying \cref{eq:pop_error} on each $\H_c$ yields:
\begin{equation}
 E_{\text{gen}}
=
\sigma^2(1+\frac{(|c|-1)}{\time_o-|c|}).
\end{equation}
For a cycle basis $C$, the total expected generalization error is:
\begin{equation}
\label{eq:cycle_basis_pop_error}
\sum\nolimits_{c\in\mathcal{C}}
\sigma^2(1+\frac{(|c|-1)}{\time_o-|c|}).
\end{equation}

Since the denominator decreases and the numerator increases with $|c|$, minimizing generalization error reduces to minimizing total cycle length $|c|$.

\begin{restatable}{coro}[Optimality of \acf{mcb}]
The \ac{mcb} \cite{kavitha2007ano}
\begin{equation}
\label{eq:MCB_optimal}
\mathcal{C}^{\mathrm{mcb}}
= \arg\min_{\mathcal{C}} \sum\nolimits_{c\in\mathcal{C}} |c|
\end{equation}
achieves the minimum generalization error among all Cycle-Space estimators.
\end{restatable}

Based on \cref{eq:MCB_optimal}, we can achieve optimal generalization error by decomposing the topology into \ac{mcb}.
The time complexity of computing \ac{mcb} is $O(\dimz^2\dimx)$ and is incurred only once under a fixed topology.
}

\section{Experiments}\label{sec:evaluation}

This section evaluates the effectiveness of the proposed framework from several complementary perspectives. 
First, we validate that the proposed autoencoder (AE) successfully learns the measurement manifold and enables the construction of stealthy blind FDIA attacks. 
Second, we evaluate the ability of the proposed Cycle-Space-informed detection framework to identify such attacks and compare its performance with existing machine learning based anomaly detectors. 
Third, we analyze several practical aspects of the framework, including decentralized detection coverage, extension to AC systems, and robustness to measurement noise. Finally, additional experimental results conducted under partial observability are provided in \Cref{sec: partial_exp}.

\textbf{Dataset Configuration.}\ 
We evaluate the proposed framework on several standard IEEE transmission grid test systems, including the 14-bus, 30-bus, 57-bus, and 118-bus systems. These benchmark networks cover a wide range of topological characteristics, from small and sparse networks (e.g., IEEE 14-bus) to larger and more interconnected systems (e.g., IEEE 118-bus). Such test cases are widely used in power system research to validate state estimation and cybersecurity methodologies.

For each system, time-series measurement data are generated using DC power flow simulations implemented in MATPOWER~\cite{zimmerman2010matpower}. To introduce realistic temporal variability, load and generation profiles are scaled using time-series data derived from the Duquesne Light Company (Pittsburgh). Random scaling factors are applied to simulate natural fluctuations in system demand. In addition, Gaussian white noise with a standard deviation of $0.02$ p.u. is added to all measurements to emulate realistic sensor noise levels commonly encountered in operational power systems.

% \textbf{Choice of Attack Methods.}\ To comprehensively evaluate the proposed detection framework, we consider not only the autoencoder-based attack developed in Section~\ref{sec:attack-AC}, but also several representative blind FDIA generation methods from the literature. Specifically, the attack baselines include
% \begin{itemize}
% \item a principal component analysis (PCA)-based Jacobian estimator~\cite{yu2015blind},
% \item a low-rank singular value decomposition (low-rank SVD) based approach~\cite{yang2023false},
% \item iAttackgen~\cite{shahriar2021iattackgen}, a generative adversarial model for crafting stealthy attacks.
% \end{itemize}

\textbf{Implementation Details.}\ 
For constructing the blind FDIA attack, we employ the autoencoder architecture described in Section~\ref{sec:AE-method}. The AE learns the measurement manifold directly from historical measurement data. The input to the AE consists of active power flow measurements across all transmission branches of the grid. For example, in the IEEE 14-bus system containing 20 branches, the input vector dimension is 20. The AE output mirrors the input and the network is trained to minimize the mean squared reconstruction error.

To ensure that the AE captures the intrinsic system behavior while filtering out measurement noise, the latent dimension is set equal to the number of independent system states. Under the DC model, this corresponds to the number of buses minus one (excluding the slack bus). For instance, the latent dimension of the IEEE 14-bus system is set to 13. This design allows the AE to learn the low-dimensional physical subspace defined by the power flow equations while suppressing noise components. Both the encoder and decoder consist of five fully connected layers with approximately ten neurons each using \ac{relu} activations.

The AE is trained for up to 300 epochs using the Adam optimizer with a learning rate of $2\times10^{-4}$ and a batch size of 50. When adversarial modules such as generators and discriminators are involved, they are updated every five iterations to maintain stable convergence. All simulations are performed using MATLAB 2022b for power flow generation and Python 3.12 for AE training and attack construction. The experiments are conducted on a workstation equipped with an Intel Core i7 2.2 GHz processor and 16 GB of RAM.

\subsection{Evaluation of Autoencoded Blind FDIA}
\label{sec:exp-AE}

We first evaluate whether the proposed autoencoder can successfully learn the underlying measurement manifold of the power system. Since the measurement space is high dimensional, direct visualization is not feasible. Instead, we select four representative branch power flow measurements in the IEEE 14-bus system, corresponding to lines $(2,3)$, $(2,5)$, $(4,7)$, and $(6,13)$, denoted as pf1 through pf4.

Figure~\ref{fig:manifold-learn-correct} illustrates the original measurements (top row) and the AE reconstructed measurements (bottom row) for several combinations of these power flow variables. The close geometric alignment between the original and reconstructed measurements demonstrates that the AE successfully captures the low-dimensional measurement manifold of the system.

This reconstruction capability is essential for the blind FDIA described in Section~\ref{sec:AE-method}. By reconstructing measurements that lie close to the learned manifold, the attacker can generate perturbations that remain near the approximate null space of the residual-based bad data detector (BDD), allowing the attack to remain stealthy.

\begin{figure}[h]
\centering
\includegraphics[width=0.8\linewidth]{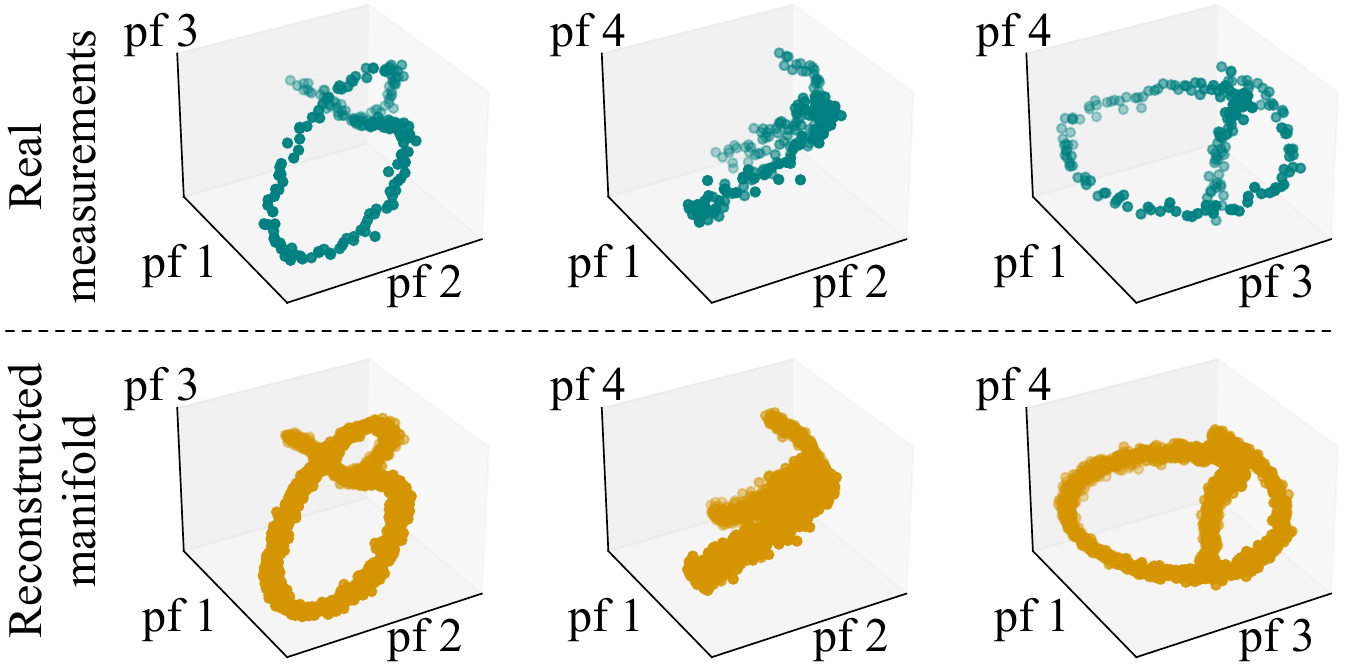}
\caption{Comparison of original and reconstructed measurement manifolds in the IEEE 14-bus system.}
\label{fig:manifold-learn-correct}
\end{figure}

\subsection{Comparison with \ac{ml}-based Detectors}
\label{sec:exp-compare-detectors}

We next evaluate the ability of the proposed Cycle-Space-informed detection framework to detect the blind FDIA attack. Figure~\ref{fig:compare-detectors} compares the proposed detector with several advanced machine learning based anomaly detection methods, including an \ac{lstm}-based detector~\cite{deepSignalAnomalyDetector} and an \acs{if} algorithm.
For each system, anomaly scores are computed for measurement sequences containing the proposed AE-based blind FDIA. The results show that machine learning based detectors frequently produce false positives due to natural system fluctuations and load variability. Consequently, these methods often fail to clearly identify the attack interval.

In contrast, the proposed Cycle-Space detector consistently identifies the attack interval across all tested systems. This behavior highlights the advantage of incorporating physical system structure into the detection process. By leveraging the Cycle-Space structure of the power network, the proposed method can isolate the null space components introduced by the attack, enabling accurate detection even when the attack is designed to evade data-driven anomaly detectors.

\begin{figure}[h]
    \centering
    \includegraphics[width= 0.9\linewidth]{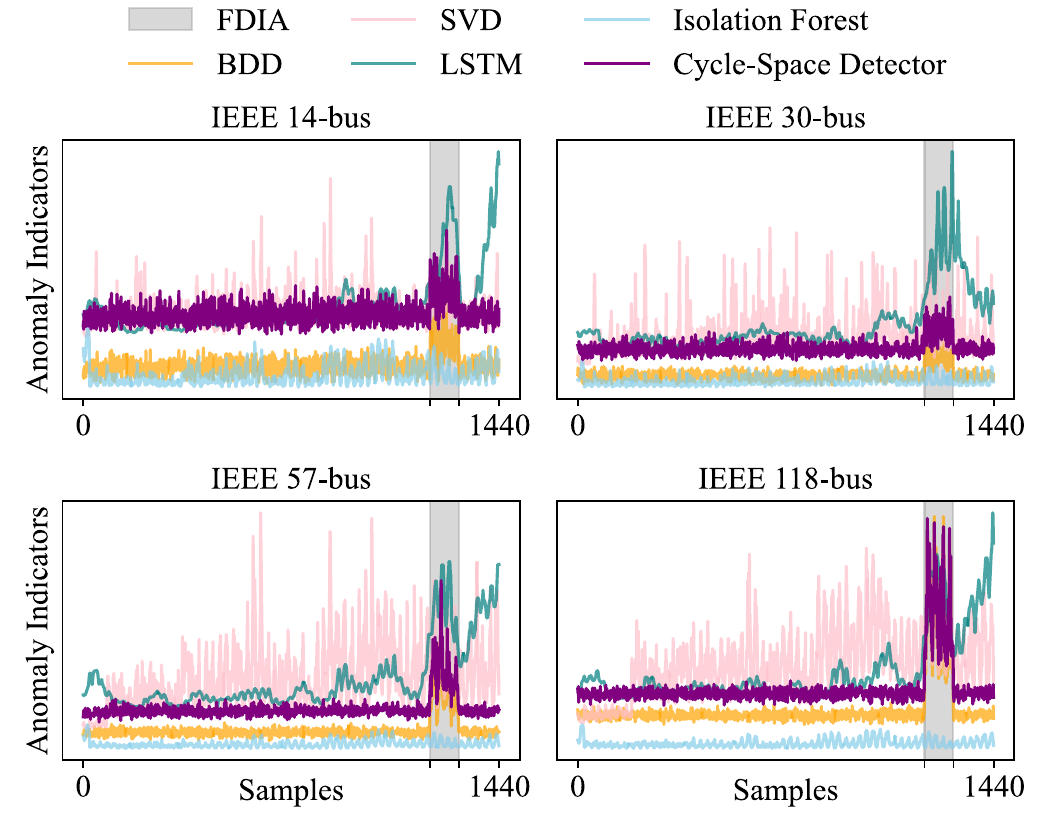}
    \caption{The comparison of our proposed Cycle-Space based detector and \ac{ml}-based detectors on the proposed FDIA in IEEE 14, 30, 57, 118-bus systems.}
    \label{fig:compare-detectors}
\end{figure}

\subsection{Aggregated vs. Single Decentralized Detection Coverage Comparison}
\label{sec:exp-agg_vs_single}

To evaluate the effectiveness of the decentralized detection architecture described in Section~\ref{sec:cycle_informed_method}, we analyze the detection capability of individual cycle-based detectors (D-detectors) and compare their performance with the aggregated detector.
Figure~\ref{fig:f1} presents the $F_1$ scores (visualized as colored heatmaps) in the attack magnitude–noise plane for the IEEE 14-, 30-, 57-, and 118-bus systems. The dashed blue curves correspond to the $F_1=0.8$ boundary achieved by individual D-detectors, while the red dashed curves correspond to the aggregated detector.

The results show that while individual cycle-based detectors can detect attacks within limited operating regions, the aggregated detector provides significantly broader detection coverage. This demonstrates that combining information from multiple cycle-based detectors enhances detection reliability and improves robustness against varying attack magnitudes and noise levels.To assess robustness when full state information is unavailable, we conduct an additional study under partial observability in \Cref{sec: partial_exp}.

\begin{figure}[h]
    \centering
    \includegraphics[width=  0.9\linewidth]{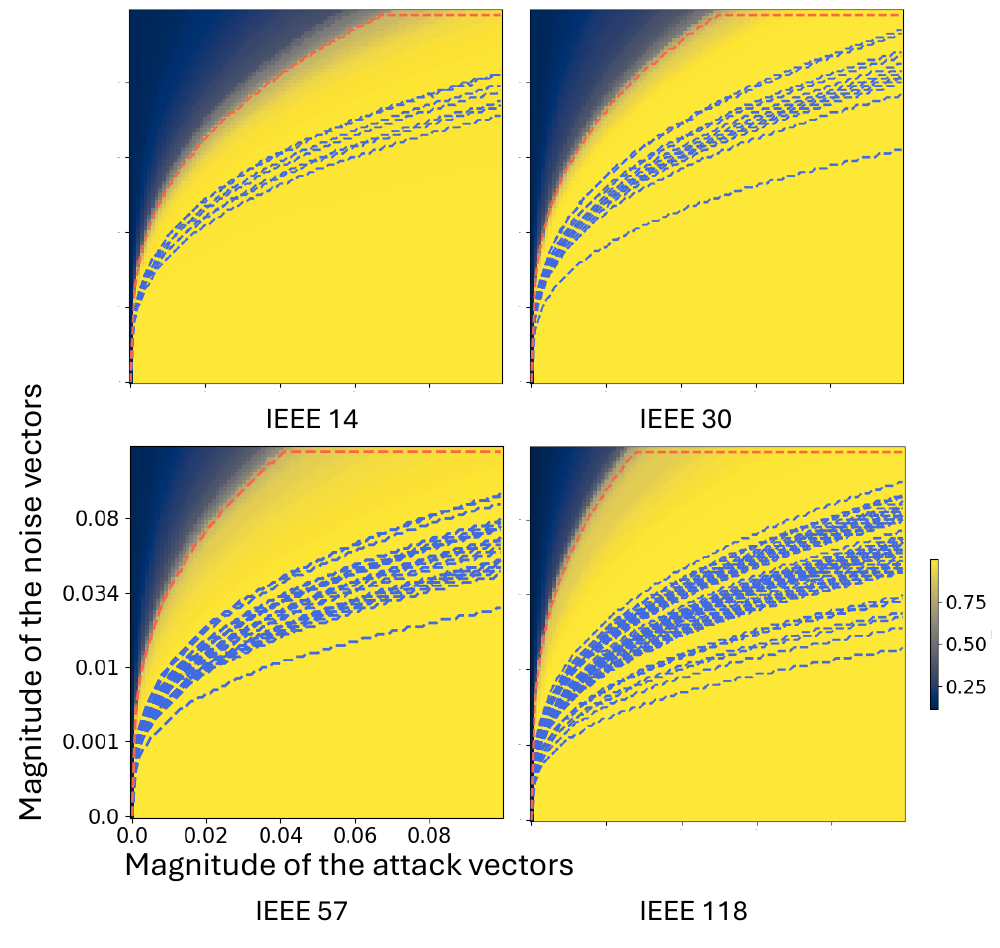}
    \caption{$F_1$ scores (colored heatmap) in the attack magnitude-noise plane, for IEEE 14, 30, 57, 118-bus systems, where the 0.8 value curves for single D-detectors (dashed blue lines) and aggregated result (red dashed line) are shown for comparison.}
    \label{fig:f1}
\end{figure}

% \subsection{Sensitivity Analysis for AE Latent Size}
% \label{sec:AE-latentdimension}

% The latent dimensionality of the autoencoder plays a critical role in accurately capturing the measurement manifold. To analyze this design choice, we vary the AE latent dimension in the IEEE 14-bus system from 1 to 20 and evaluate its impact on reconstruction and residual errors.
% %
% The IEEE 14-bus system contains 13 independent system states due to the fixed slack bus. Therefore, the latent dimension should ideally match this value. Figure~\ref{fig:errors_14bus} (top) shows the mean residual error (green line) and reconstruction error as functions of the latent dimension.

% As the latent dimension increases, the residual error decreases approximately linearly. When the latent dimension approaches 20 (the number of available measurements), the residual error converges to the level observed in real measurements (yellow line). These results confirm that selecting the latent dimension to match the intrinsic system state dimension allows the AE to accurately capture the physical subspace of the measurements while suppressing noise components.

% \begin{figure}[H]
% \centering
% \includegraphics[width=0.95\linewidth]{images/hidden.pdf}
% \caption{Residual error and reconstruction error against hidden layer size for IEEE 14-bus test case.}
% \label{fig:errors_14bus}
% \end{figure}

\subsection{Sensitivity Analysis under Noise}

We now evaluate the robustness of the Cycle-Space detection framework under different measurement noise levels. Figure~\ref{fig:noise_sensitive} illustrates the reconstruction loss obtained during detection for varying standard deviations of measurement noise in the IEEE 14-bus system. The shaded gray region corresponds to the interval during which the AE-based blind FDIA is active.

\begin{figure}[h]
    \centering
    \vskip -0.1in
    \includegraphics[width= 0.8\linewidth]{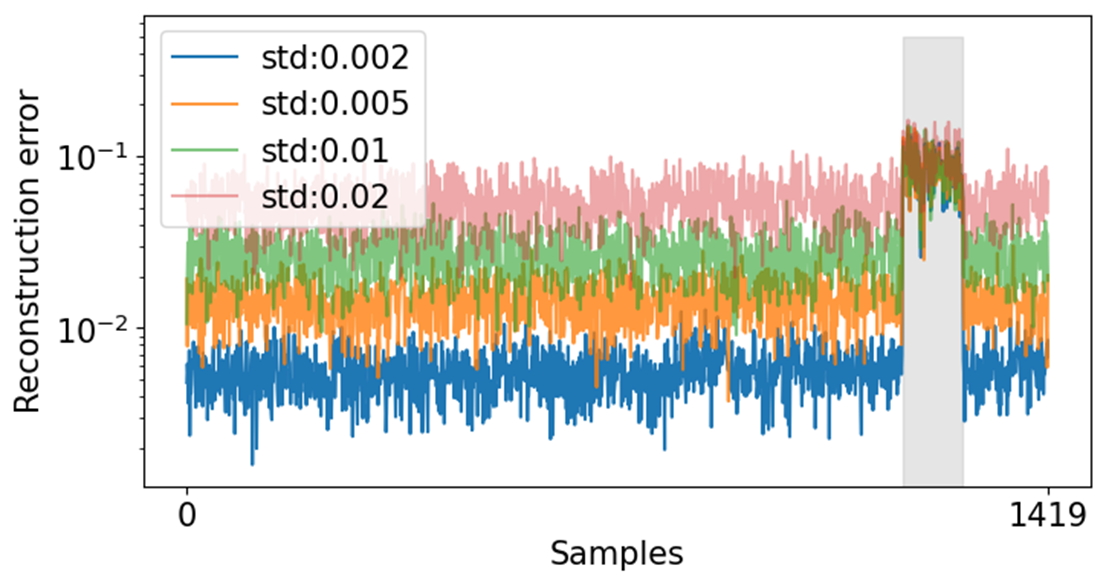}
    \vskip -0.1in
    \caption{Reconstructed loss in detection procedure in various noise levels in IEEE 14-bus test case.}
    \label{fig:noise_sensitive}
    \vskip -0.1in
\end{figure}

The results show that the proposed detector maintains reliable detection capability under realistic PMU noise levels (approximately $0.01$ p.u.). This indicates that the Cycle-Space detection approach remains effective even when measurements contain realistic levels of sensor noise, supporting its applicability in practical power system monitoring environments.

\subsection{Extension to Nonlinear AC Systems}
% To assess the generalizability of our framework beyond linear DC models, we evaluate the performance of \ac{csd} in AC systems under data-driven \ac{fdia}. In this setting, $\z$ includes both active and reactive power flow measurements based on the AC power flow equations. We compare the detection results of \ac{csd} and \ac{ml}-based methods in Fig.~\ref{fig:detection_comparison_AC}. Similar to the DC setting, \ac{ml}-based anomaly detectors fail to reliably detect the proposed AC attack. Specifically, the \acs{if} produces no significant response, while the \ac{lstm}-based detector exhibits erratic behavior, with inconsistent localization and high false positive rates. These results confirm the vulnerability of \ac{ml}-based anomaly detectors in such FDIAs. Since such models operate purely on statistical patterns without leveraging the system’s physical structure, their performance is  influenced by the consistency of the observed measurement trajectory. By locally adjusting $\boldsymbol{J}$ (see Section \ref{sec:attack-AC}) to maintain the attack along the manifold, the adversary ensures that the perturbation remains statistically consistent with normal behavior, further underscoring the advantage of geometry-aware detectors like \ac{csd}.

To test whether our framework works beyond the linear DC model, we run \ac{csd} on AC systems facing a data-driven \ac{fdia}. Here, $\z$ contains the active and reactive power flow measurements that follow the AC power flow equations. Figure~\ref{fig:detection_comparison_AC} shows the detection results of \ac{csd} next to the \ac{ml}-based methods. As in the DC case, the \ac{ml}-based detectors cannot reliably identify the AC attack. The \acs{if} gives almost no response, and the \ac{lstm}-based detector behaves unevenly, flagging the wrong time window and raising many false alarms. This again shows that \ac{ml}-based detectors are weak against these FDIAs. Because these models depend only on statistical patterns and ignore the physical layout of the grid, their accuracy hinges on whether the measured signal stays steady over time. By updating $\boldsymbol{J}$ locally (see Section~\ref{sec:attack-AC}) so the attack stays on the manifold, the attacker keeps the injected change looking like normal data, which once more points to the value of structure-aware detectors such as \ac{csd}.

\begin{figure}[h]
    \centering
    \vskip -0.1in
    \includegraphics[width=0.8\linewidth]{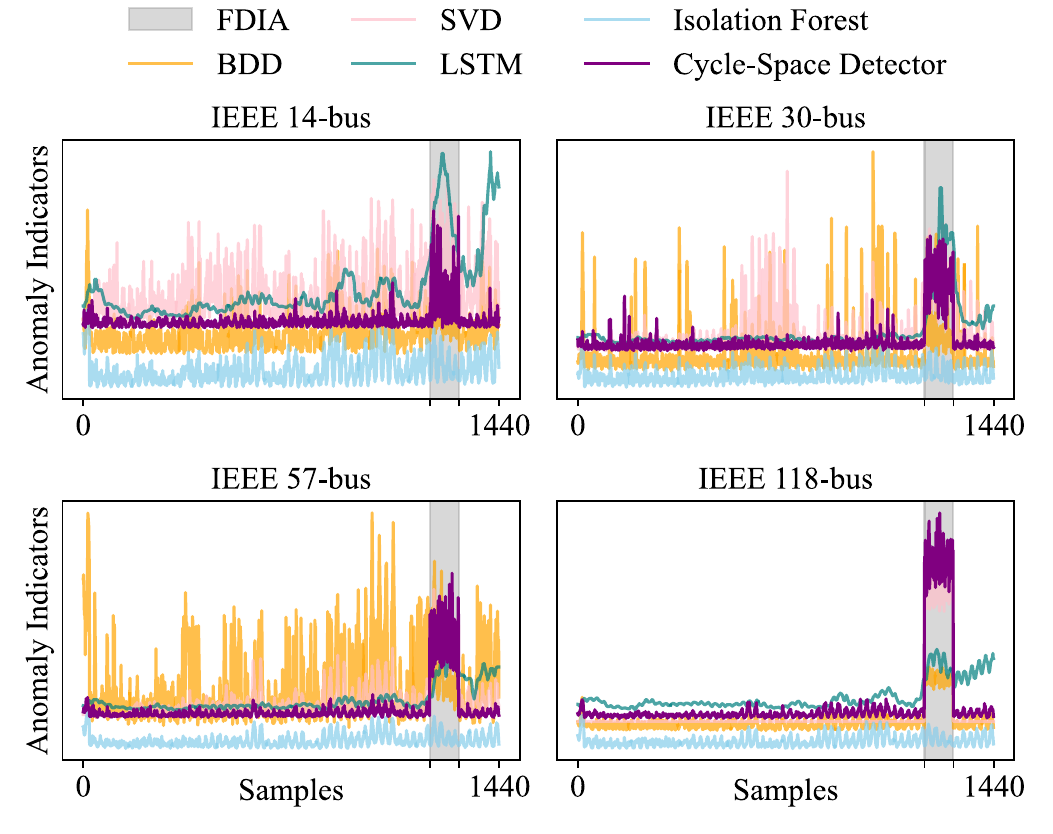}
    \vskip -0.1in
    \caption{The comparison of \ac{csd} and \ac{ml}-based detectors on the proposed FDIA in IEEE $14$, $30$, $57$, $118$-bus AC systems.}
    \vskip -0.1in
    \label{fig:detection_comparison_AC}
\end{figure}

% In contrast, our proposed \ac{csd} remains robust and consistently detects the attack period across all AC systems. While performance is slightly degraded compared to the DC case, the attack window remains clearly identifiable. This degradation arises from a key structural difference between DC and AC systems: in DC models, the measurements lie on a low-dimensional flat subspace, resulting in a sharp drop in singular values after the true rank.
By contrast, our \ac{csd} stays reliable and picks out the attack period in every AC system we tested. Its performance drops a little relative to the DC case, yet the attack window is still easy to spot. This drop comes from a basic structural gap between DC and AC systems: under the DC model the measurements sit on a low -dimensional flat subspace, so the singular values fall off sharply once the true rank is passed.
In AC systems, however, the measurements lie on a nonlinear low-dimensional manifold, causing the singular values to decay more gradually rather than collapsing sharply \footnote{ As for a comparison between AC and DC \ac{state-estimation} results, since we use a linearized AC description, it is expected to exhibit similar behavior as in the DC case and is therefore omitted for brevity.}. As a result, the distinction between normal and anomalous subspaces becomes less pronounced. To account for this effect, the singular value threshold in \ac{csd} must be tuned appropriately. 

% We find that setting the SVD threshold $\mathcal{T}$ to values between $10^{-1}$ and $10^{-2}$ yields robust and consistent performance across all AC systems. This adjustment allows the method to effectively separate the underlying structure of the measurement data from the perturbed subspace induced by the attack. These results confirm that the geometric foundation of \ac{csd} extends effectively to nonlinear AC systems, maintaining high detection accuracy and low false positive rates despite manifold curvature and the departure from exact linear structure, and outperforming both classical and learning-based detectors.

\begin{table}[ht]
\centering
\caption{Detection performance F1-score(\%) across IEEE test systems and attack models. Detectors include \acs{svd}, \ac{bdd}, \ac{if}, \ac{lstm}, and \ac{csd}. $M^2$ refers to the proposed Autoencoder-based (i.e., Measurement Manifold learning-based) \ac{fdia}.}
\label{tab:detection_results}
\scriptsize
\begin{tabular}{@{}llccccc@{}}
\toprule
\textbf{System} & \textbf{Attack} & \textbf{SVD} & \textbf{BDD} & \textbf{Iso. Forest} & \textbf{LSTM} & \textbf{\ac{csd}} \\ \midrule
\multirow{4}{*}{IEEE 14}  & PCA        & 31.7 & 38.3 & 15.5 & 34.0 & \textbf{50.4} \\
                         & Low-rank SVD        & 31.8 & 34.6 & 16.6 & \textbf{56.2} & 48.6 \\
                         & iAttackgen & 31.6 & 34.9 & 15.5 & \textbf{56.2} & 49.6 \\
                         & $M^2$        & 30.2 & 27.8 & 15.2 & 43.4 & \textbf{44.2} \\ \midrule
\multirow{4}{*}{IEEE 30}  & PCA        & 60.2 & 42.8 & 15.1 & 60.2 & \textbf{95.7} \\
                         & Low-rank SVD        & 58.5 & 45.2 & 13.9 & 57.5 & \textbf{94.7} \\
                         & iAttackgen & 60.2 & 50.0 & 15.1 & 57.8 & \textbf{94.1} \\
                         & $M^2$       & 57.0 & 19.7 & 13.5 & 57.8 & \textbf{91.8} \\ \midrule
\multirow{4}{*}{IEEE 57}  & PCA        & 74.3 & 17.9 & 17.5 & 55.6 & \textbf{98.5} \\
                         & Low-rank SVD        & 74.3 & 17.5 & 16.2 & 51.7 & \textbf{98.5} \\
                         & iAttackgen & 76.0 & 18.6 & 17.8 & 52.6 & \textbf{99.0} \\
                         & $M^2$       & 71.7 & 16.1 & 17.4 & 43.5 & \textbf{98.5} \\ \midrule
\multirow{4}{*}{IEEE 118} & PCA        & 99.0 & 73.8 & 17.0 & 54.1 & \textbf{100} \\
                         & Low-rank SVD        & 99.0 & 75.2 & 16.7 & 56.2 & \textbf{100} \\
                         & iAttackgen & 99.0 & 75.2 & 16.4 & 53.6 & \textbf{100} \\
                         & $M^2$       & 94.3 & 66.2 & 16.5 & 51.9 & \textbf{99.5} \\ \bottomrule
\end{tabular}
\end{table}

To further evaluate the proposed \ac{csd} detection framework, we consider not only the autoencoder-based attack developed in Section~\ref{sec:attack-AC}, but also several representative blind FDIA generation methods from the literature. Specifically, we consider the following attacks for comparison:
\begin{itemize}
\item a principal component analysis (PCA)-based Jacobian estimator~\cite{yu2015blind},
\item a low-rank singular value decomposition (low-rank SVD) based approach~\cite{yang2023false},
\item iAttackgen~\cite{shahriar2021iattackgen}, a generative adversarial model for crafting stealthy attacks.
\end{itemize}
As illustrated in \Cref{tab:detection_results}, first, the proposed autoencoder attack, $M^2$, achieves lower BDD detection rates (or equivalently, higher BDD bypass rates) than several existing blind FDIA approaches. Second, the \ac{csd} exhibits superior detection capabilities compared to \ac{svd}, \ac{bdd}, \ac{if}, and \ac{lstm}. 
Specifically, \ac{csd} maintains robust performance across all four IEEE systems regardless of the attack model employed.
Finally, the improved \ac{csd} performance observed in larger systems suggests that higher topological complexity provides a richer set of cycle based constraints, which significantly enhances the fidelity of the null space estimation.

\section{Conclusion and future work} \label{sec:conclusion}

This paper investigates blind false data injection attacks (FDIAs) and corresponding detection mechanisms in power system state estimation. We first propose a data-driven FDIA construction method based on an autoencoder (AE) that learns the measurement manifold directly from historical data. By exploiting the orthogonal direction of this manifold, the attack can generate stealthy perturbations that bypass traditional bad data detection (BDD) without requiring prior knowledge of system parameters. Numerical results demonstrate that the proposed AE-based attack achieves higher BDD bypass rates than several existing blind FDIA approaches.
To counter such attacks, we develop a topology-aware \ac{csd} based on the Cycle-Space of the network, { which is shown to achieve optimal generalization error among Cycle-Space estimators associated with the attack detection.
A primary contribution of this research is the first use of Cycle-Space theory in the context of cyber-physical security.
This structural approach provides a fundamental advantage over purely data-driven detectors by grounding security in the invariant physical topology of the grid.}
By leveraging graph structural information, the proposed method reconstructs the system null space without relying on precise line parameters and enables decentralized detection across network cycles.
Experimental results on multiple IEEE benchmark systems show that the proposed detector consistently identifies the proposed blind FDIA, as well as competitive/existing blind FDIAs, and outperforms several machine-learning-based anomaly detectors while maintaining robustness under measurement noise. 
As a supporting analytical result, we derive a first-order finite-sample expression for the expected generalization error of the null space estimator, which we use to establish the optimality of the Minimum Cycle Basis among cycle-basis-based estimators. 
The expression depends only on $\text{rank}(\H)$, sample size, and noise variance, and applies to subspace estimation beyond power systems.

Future work will focus on expanding the applicability of the Cycle-Space formalism, particularly in topology estimation and in the optimization of moving target defense strategies.

\bibliographystyle{IEEEtran}
\bibliography{mybib}

% \newpage

\appendix

\renewcommand{\theequation}{A.\arabic{equation}}
\setcounter{equation}{0}

\subsection{Expected Generalization Error Simplification}

\label{sec:gen_error_simle}
$\Z_\star^T\hat{\n}$ can be expanded as:
\begin{equation}
\Z_\star^T\hat{\n}=\Z_\star^T\n+\Z_\star^T\err=\bar{\Z}_\star^T\n+\E_\star^T\n+\Z_\star^T\err. 
 \end{equation}
Since by definition $\bar{\Z}_\star^T\n=0$:
\begin{equation}
 \Z_\star^T\hat{\n}=\E_\star^T\n+\Z_\star^T\err 
 \end{equation}
And therefore: 
\begin{align}
\|\Z_\star^T\hat{\n}\|_2^2&=\|\E_\star^T\n+\Z_\star^T\err\|_2^2 \nonumber\\
&= (\E_\star^T\n+\Z_\star^T\err)^T(\E_\star^T\n+\Z_\star^T\err) \nonumber\\ 
&=\|\E_\star^T\n\|^2+2(\E_\star^T\n)^T\Z_\star^T\err+\|\Z_\star^T\err\|^2
\end{align}  

Now we use this expanded term in~\cref{eq:gen_error} (the definition of $ E_{\text{gen}}$):
\begin{align}
   &E_{\text{gen}}=\notag\\&\lim_{\time_\star\rightarrow\infty}\frac{1}{\time_\star}\EXP_\Z[\|\E_\star^T\n\|^2+2(\E_\star^T\n)^T\Z_\star^T\err+\|\Z_\star^T\err\|^2]
\end{align}

Since $\frac{1}{\time_\star}\EXP_\Z[\|\E_\star^T\n\|^2]=\sigma^2$:
\begin{equation}
\label{final_empirical}
 E_{\text{gen}}=\sigma^2+\lim_{\time_\star\rightarrow\infty}\frac{1}{\time_\star}\EXP_\Z[2(\E_\star^T\n)^T\Z_\star^T\err+\|\Z_\star^T\err\|^2]
 \end{equation}

Next we focus on $\EXP_\Z[2(\E_\star^T\n)^T\Z_\star^T\err]$:

\begin{align}
\EXP_\Z[(\E_\star^T\n)^T\Z_\star^T\err]&=\EXP_\Z[\n^T\E_\star\Z_\star^T\err] \\
&= \EXP_\Z[\n^T\E_\star(\bar{\Z}_\star^T+\E_\star^T)\err]
\end{align}

Since $\E_\star,\n,\bar{\Z}_\star,\err$ are independent we use the expectation separability property:
\begin{equation}
    \EXP_\Z[\n^T\E_\star \bar{\Z}_\star^T\err]= \EXP_\Z[\n^T]\EXP_\Z[\E_\star]\EXP_\Z[\bar{\Z}_\star^T\err]
\end{equation}
Observing that $\EXP_\Z[\E_\star]=0_{m\times\time_\star}$, we therefore have:

\begin{equation}
\label{eq:simple_generalization}
E_{\text{gen}}=\sigma^2+\lim_{\time_\star\rightarrow\infty}\frac{\EXP_\Z[\|\Z_\star^T\err\|^2+2\n^T\E_\star \E_\star^T\err]}{\time_\star}
 \end{equation}

\subsection{First-order Lagrangian Solution}
\label{sec:lagrangian_proof}

\begin{prop}[First Order Lagrangian Solution]
\label{eq:prop_delta_solution_a}
The first order solution for the vector $\err$ which minimizes $\|\Z^T\hat{\n} \|^2$ is given by:
\begin{equation}
\begin{aligned}
    \err =& -(\bar\Z^T)^{+} \E^T \n
\end{aligned}
\label{eq:delta_solution_a}
\end{equation}
\end{prop}

\begin{proof}
When the noise std $\sigma$ is small, we may safely disregard the second-order term $\E^T \err$.

Also:
\begin{align}
    \|\n+\err\|^2&=(\n+\err)^T(\n+\err)\\
    &= \n^T\n+2\n^T\err+\err^T\err\\
    &= 1+2\n^T\err+\err^T\err\\
    &= 1
\end{align}
Disregarding the second order term $\err^T\err$, we have $\n^T\err=0$.
So the optimization problem becomes:
\begin{equation}
    \label{eq:small_noise_null_estimation_error}
    \min_{\err} \frac{1}{2} \|\bar\Z^T\err + \E^T \n\|^2 \quad \text{s.t.} \quad \n^T\err = 0
\end{equation}

Consider the following Lagrangian:
\begin{equation}
    \mathcal{L}(\err, \mu) = \frac{1}{2} (\bar \Z^T \err + \E^T \n)^T (\bar\Z^T \err + \E^T \n) + \mu (\n^T \err)
\end{equation}

Taking the derivative with respect to $\err$ (First-Order Condition):

\begin{equation}
    \frac{\partial \mathcal{L}}{\partial \err} = \bar\Z(\bar\Z^T \err + \E^T \n) + \mu \n = 0
\end{equation}
\begin{equation}
    \bar\Z\bar\Z^T \err + \bar\Z\E^T \n + \mu \n = 0
    \label{eq:l_constrain}
\end{equation}

Solving for the Lagrange Multiplier $\mu$, we left-multiply by $\n^T$:
\begin{equation}\n^T\bar\Z \bar\Z^T \err + \n^T \bar\Z \E^T \n + \mu (\n^T \n) = 0
\end{equation}
 
Applying the following properties: 1) Orthogonality: Since $\bar\Z^T \n = 0$, taking the transpose gives $\n^T \bar\Z = 0$. 2) Unit Vector: $\n^T \n = 1$. Substituting these expressions into the above equation:
\begin{equation}(0) \bar\Z^T \n + (0) \E^T \n + \mu (1) = 0\end{equation}
\begin{equation}\implies \mu = 0\end{equation}

Substituting $\mu = 0$ back into \cref{eq:l_constrain}:
\begin{equation}\bar\Z\bar\Z^T \err + \bar\Z \E^T \n = 0\end{equation}
\begin{equation}\bar\Z \bar\Z^T \err = - \bar\Z \E^T \n\end{equation}
Finally, we now isolate $\err$. The matrix $\bar\Z\bar\Z^T$ may not be invertible (singular), so we use the Moore-Penrose pseudoinverse, denoted as $(\bar\Z\bar\Z^T)^+$.
\begin{align}
    \err = -(\bar\Z\bar\Z^T)^{+}\bar\Z \E^T \n+(I_\dimz-(\bar\Z\bar\Z^T)^+\bar\Z\bar\Z^T)y, \\ y\quad \text{arbitrary}\in R^m\end{align}

Since $(I_m-(\bar\Z\bar\Z^T)^+\bar\Z\bar\Z^T)y$ lies in the null space of $\bar\Z^T$, we disregard this term and $(\bar\Z \bar\Z^T)^+ \bar\Z = (\bar\Z^T)^+$.

Thus, the simplified first-order solution for $\err$ is 
\begin{equation}
\err = -(\bar\Z^T)^{+} \E^T \n
\end{equation}

\end{proof}

\subsection{Expected Generalization Error}
\label{sec:gen_error_proof}

\begin{theorem}[First Order Finite Sample Expected Generalization Error]
The first-order finite sample expected generalization error is given by:
\begin{equation}
\boxed{
 E_{\text{gen}}=\sigma^2(1+
 \frac{\mathrm{rank}(\H)}{\time_o-n})
}
\label{eq:pop_error_a}
\end{equation}
\label{eq:theo_pop_error_a}
\end{theorem}

\begin{proof}

In \cref{eq:final_generalization}, the Expected Generalization Error is: 
\begin{equation}
E_{gen}=\sigma^2+\lim_{\time_\star\rightarrow\infty}\frac{\EXP_\Z[\|\Z_\star^T\err\|^2+2\n^T\E_\star \E_\star^T\err]}{\time_\star}
\end{equation}

Plugging \cref{eq:delta_solution} into  $\EXP_\Z[\n^T\E_\star \E_\star^T\err]$:

\begin{align}
    &\EXP_\Z[\n^T\E_\star \E_\star^T\err]\\&=\EXP_\Z[\n^T\E_\star \E_\star^T(\bar\Z^+)^T\E^T\n]\\
    &= Tr(\EXP_\Z[\E_\star \E_\star^T(\bar\Z^+)^T\E^T\n\n^T])\\
    &= Tr(\EXP_\Z[\E_\star \E_\star^T]\EXP_\Z[(\bar\Z^+)^T]\EXP_\Z[\E^T\n\n^T])
\end{align}

Because $\EXP_\Z[\E^T\n\n^T]=0_{\time_o\times m}$, $\EXP_\Z[\n^T\E_\star \E_\star^T\err]=0$.

Expanding $\frac{\EXP_\Z[\|\Z_\star^T\err\|^2]}{\time_\star}$:

\begin{align}
     \frac{\|\Z_\star^T \err\|^2}{\time_\star}&=\frac{\|(\bar{\Z}_{\star}+\E_\star)^T \err\|^2}{\time_\star}\\
     &=\frac{\err^T(\bar{\Z}_{\star}+\E_\star)(\bar{\Z}_{\star}+\E_\star)^T\err }{\time_\star}\\
     &=\frac{\err^T(\bar{\Z}_{\star}\bar{\Z}_{\star}^T+\E_\star \bar{\Z}_\star^T+\bar{\Z}_\star \E_\star^T+\E_\star \E_\star^T)\err }{\time_\star}\\
     &=\frac{\err^T(\H \X_\star  \X_\star ^T\H^T+\time_\star I_m\sigma^2)\err}{\time_\star}
\end{align}

Plugging \cref{eq:delta_solution} into $\sigma^2\err^TI_m\err$:

\begin{align}
\label{eq:ignore_error2}
\sigma^2\err^TI_m\err=\sigma^2\n^T\E \bar\Z^+(\bar\Z^+)^T\E^T\n
\end{align}

Since:
\begin{align}
\sigma^2\err^TI_m\err=\sigma^4\frac{Tr((\H\H^T)^+)}{\time_o-n-1}=\mathcal{O}(\sigma^4\frac{\|\H^+\|^2}{\time_o-n-1})
\end{align}

Since $\lim_{\time_\star\rightarrow\infty} \frac{ \X_\star  \X_\star ^T}{\time_\star}= I_n$:

\begin{equation}
 \lim_{\time_\star\rightarrow\infty}\frac{\|\Z_\star^T \err\|^2}{\time_\star}
= \err^T \H\H^T \err +\mathcal{O}(\sigma^4\frac{\|\H^+\|^2}{\time_o-n-1})
 \end{equation}
Using the cyclic property of the trace \cite{WikipediaTrace}:
\begin{equation}
\err^T \H\H^T \err
= \mathrm{Tr}(\H\H^T\,\err\err^T).
\end{equation}

Taking expectation:
\begin{align}
\EXP_\Z[\err^T \H\H^T \err]
= \mathrm{Tr}\left(\H\H^T \EXP_\Z[\err\err^T]\right)\\
= \mathrm{Tr}\left(\H\H^T \mathrm{Cov}(\err)\right).
\end{align}

Using the covariance result in \cref{cor:cov_err}:
\begin{equation}
\EXP_\Z[\err^T \H\H^T \err]
= \frac{\sigma^2}{\time_o-n}
\mathrm{Tr}\left(\H\H^T (\H\H^T)^+\right).
\end{equation}

Since $\H\H^T (\H\H^T)^+ = P_{\mathrm{Col}(\H)}$ which is the projection matrix on the column space of $\H$ \cite{swewart1977perturbation}, its trace equals the rank:
\begin{equation}
\EXP_\Z[\err^T \H\H^T \err]
= \frac{\sigma^2\,\mathrm{rank}(\H)}{\time_o-n}.
\end{equation}
So the Expected Generalization Error is :

\begin{equation}
 E_{\text{gen}}=\sigma^2(1+
 \frac{\mathrm{rank}(\H)}{\time_o-n})+\mathcal{O}(\sigma^4\frac{\|\H^+\|^2}{(\time_o-n-1)\time_o}).
\end{equation}

\end{proof}

\subsection{Covariance of the First-Order Estimation Error}
\label{sec:cov_err}

\begin{corollary}[Covariance of the First-Order Estimation Error]
\label{cor:cov_err}
Under the setting of \cref{eq:prop_delta_solution_a}, assume
$\bar{\Z}=\H\X$ with the columns of $\X$ drawn i.i.d.\ from
$\mathcal{N}(0,I_\dimx)$, $\mathrm{rank}(\H)=\dimx-1$, and $\time_o>\dimx$.
Then $\err$ has zero mean and covariance
\begin{equation}
\label{eq:cov_err_result}
\mathrm{Cov}(\err)
\;=\;
\frac{\sigma^2}{\time_o-\dimx}\,(\H\H^T)^{+}.
\end{equation}
\end{corollary}

\begin{proof}
Let $\A=-(\bar{\Z}^T)^{+}$. Then \eqref{eq:delta_solution_a} gives
\begin{equation}
\err = \A\,\E^T\n.
\end{equation}
Since the entries of $\E$ are i.i.d.\ $\mathcal{N}(0,\sigma^2)$ and $\n$
is a deterministic unit vector,
\begin{equation}
\label{eq:enn_iso}
\EXP_{\E}\left[\E^T\n\,\n^T\E\right]=\sigma^2 I_{\time_o}.
\end{equation}
Because $\bar{\Z}$ and $\E$ are independent and $\EXP[\E]=0$, the mean
$\EXP[\err]=0$, and
\begin{equation}
\label{eq:cov_to_AAT}
\mathrm{Cov}(\err)
=\EXP\left[\A\,\sigma^2 I_{\time_o}\,\A^T\right]
=\sigma^2\,\EXP\left[\A\A^T\right].
\end{equation}
Note:
\begin{equation}
\label{eq:AAT_pinv}
\A\A^T=(\bar{\Z}^T)^{+}\big((\bar{\Z}^T)^{+}\big)^{T}=(\bar{\Z}\bar{\Z}^T)^{+}.
\end{equation}
By plugging \eqref{eq:AAT_pinv} into \eqref{eq:cov_to_AAT}, we get:
\begin{equation}
\label{eq:cov_zz_pinv}
\mathrm{Cov}(\err)=\sigma^2\,\EXP\left[(\bar{\Z}\bar{\Z}^T)^{+}\right].
\end{equation}

Since $\bar{\Z}=\H\X$, we have
\begin{equation}
\bar{\Z}\bar{\Z}^T=\H(\X\X^T)\H^T.
\end{equation}
Take the SVD $\H=\mathbf{U}_r\mathbf{\Sigma}_r\mathbf{V}_r^T$ with
$\mathbf{U}_r\in\R^{\dimz\times(\dimx-1)}$, $\mathbf{\Sigma}_r\in\R^{(\dimx-1)\times(\dimx-1)}$
diagonal and invertible, and $\mathbf{V}_r^T\mathbf{V}_r=I_{\dimx-1}$. Let
$\W=\mathbf{V}_r^T\X\in\R^{(\dimx-1)\times\time_o}$. Since the columns of $\X$
are i.i.d.\ $\mathcal{N}(0,I_\dimx)$ and $\mathbf{V}_r^T\mathbf{V}_r=I_{\dimx-1}$, the
columns of $\W$ are i.i.d.\ $\mathcal{N}(0,I_{\dimx-1})$, so
$\W\W^T\sim\mathrm{Wishart}_{\dimx-1}(\time_o,I_{\dimx-1})$ and is
invertible almost surely for $\time_o>\dimx-1$. Hence
\begin{equation}
\bar{\Z}\bar{\Z}^T=\mathbf{U}_r\mathbf{\Sigma}_r\,\W\W^T\,\mathbf{\Sigma}_r\mathbf{U}_r^T,
\end{equation}
and, since $\mathbf{U}_r$ has orthonormal columns and
$\mathbf{\Sigma}_r\W\W^T\mathbf{\Sigma}_r$ is invertible,
\begin{equation}
\label{eq:zz_pinv_svd}
(\bar{\Z}\bar{\Z}^T)^{+}
=\mathbf{U}_r\mathbf{\Sigma}_r^{-1}(\W\W^T)^{-1}\mathbf{\Sigma}_r^{-1}\mathbf{U}_r^T.
\end{equation}
The inverse-Wishart mean \cite{gupta2018matrix} gives
\begin{equation}
\label{eq:iw_mean}
\EXP\left[(\W\W^T)^{-1}\right]=\frac{1}{\time_o-\dimx}\,I_{\dimx-1}.
\end{equation}
Plugging \eqref{eq:iw_mean} into the expectation of
\eqref{eq:zz_pinv_svd}:
\begin{equation}
\label{eq:E_zz_pinv}
\EXP\left[(\bar{\Z}\bar{\Z}^T)^{+}\right]
=\frac{1}{\time_o-\dimx}\,\mathbf{U}_r\mathbf{\Sigma}_r^{-2}\mathbf{U}_r^T
=\frac{1}{\time_o-\dimx}\,(\H\H^T)^{+},
\end{equation}
where the last equality uses the SVD form $(\H\H^T)^{+}=\mathbf{U}_r\mathbf{\Sigma}_r^{-2}\mathbf{U}_r^T$.
By plugging \eqref{eq:E_zz_pinv} into \eqref{eq:cov_zz_pinv}, we obtain
\eqref{eq:cov_err_result}.
\end{proof}

\subsection{Expected Generalization Error Validation}
\label{sec:gen_error_vali}
We empirically validate \cref{eq:pop_error} on the IEEE 14, 30, 57, and 118-bus systems. In each Monte Carlo trial, we generate $\time_o=2\dimz$ training and $\time_\star=1000$ test measurements under the DC model with i.i.d.\ Gaussian noise of standard deviation $\sigma$, estimate $\hat{\n}$ from the training set, and compute $\tfrac{1}{\time_\star}\EXP_{\Z}[\|\Z_\star^T\hat{\n}\|^2]$, averaged over 100 trials per $\sigma$. \Cref{fig:Egen14Valid} shows close agreement between theory and experiment across all four systems for $\sigma\leq 1$, two orders of magnitude above operational PMU noise. The crosses mark the first $\sigma$ at which the relative deviation exceeds 10\%, occurring only at $\sigma>1$ where the dropped second-order terms $\E^T\err$ and $\err^T\err$ become non-negligible.
\begin{figure}
    \centering
    \includegraphics[width=0.85\linewidth]{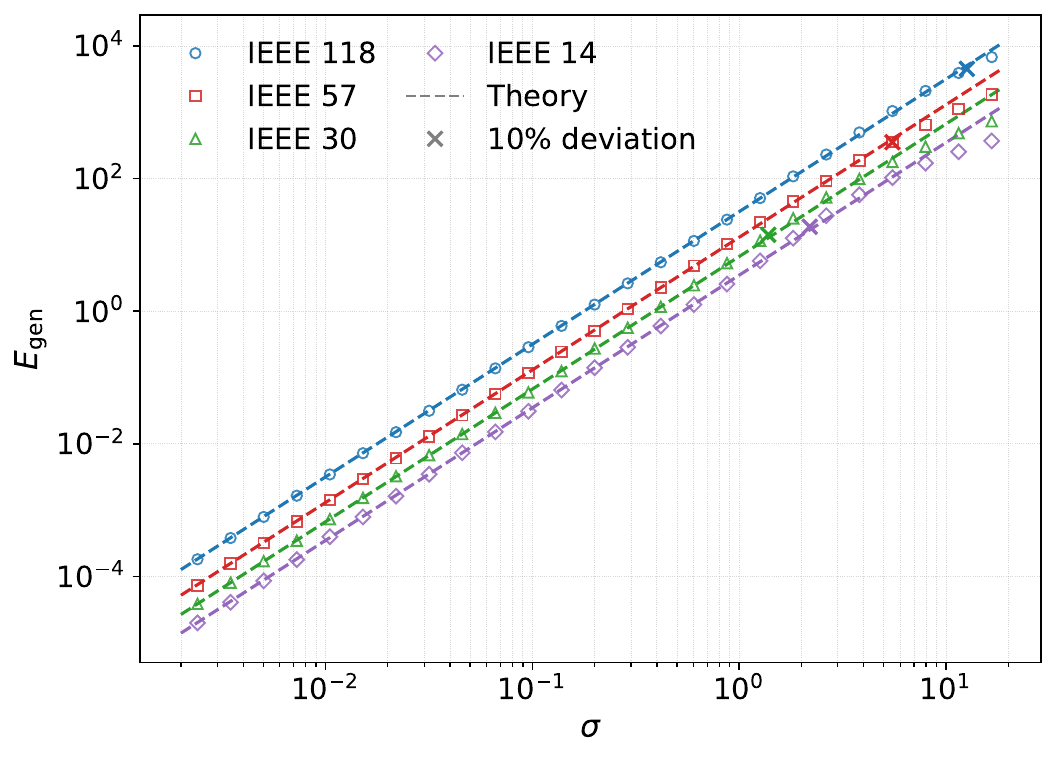}
    \caption{Validation of \cref{eq:pop_error} on IEEE 14, 30, 57, and 118-bus systems. Dashed lines: first-order prediction; markers: empirical $E_{\mathrm{gen}}$ over 100 Monte Carlo trials. Crosses indicate the first $\sigma$ at which the relative deviation exceeds 10\%.}
    \label{fig:Egen14Valid}
\end{figure}

\subsection{Partial observability analysis}
\label{sec: partial_exp}
In network monitoring, communication latencies or sensor failures often result in partial observability, where specific meter information is lost. 
This degradation in data quality directly impacts the detector's performance, as the efficacy of attack detection relies on the precision of null space estimation.
The null space is fundamentally characterized by the network's cycle-space. When link measurements are missing, the corresponding cycles are broken, leading to an incomplete cycle-space, which in turn diminishes the accuracy of the null space estimation. 
According to the generalization error defined in~\cref{eq:pop_error}, the estimation error for a cycle basis is linearly correlated with the cycle length. 
Consequently, shorter cycles yield higher estimation accuracy. 
Therefore, the loss of these small, high-fidelity cycles may reduce the detector’s ability to accurately characterize the null space.

The impact of partial observability on the \ac{csd} performance can be categorized into the following scenarios:

\begin{figure}[htbp]
     \centering
     % Row 1: A and B
     \begin{subfigure}[b]{0.23\textwidth}
         \centering
         \includegraphics[width=\textwidth]{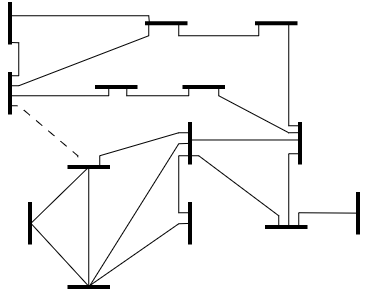}
         \caption{Partially observed IEEE 14-bus system: unobserved links (dashed) are shown for a large cycle.}
         \label{fig:remove_large}
     \end{subfigure}
     \hfill % Adds spacing between A and B
     \begin{subfigure}[b]{0.23\textwidth}
         \centering
         \includegraphics[width=\textwidth]{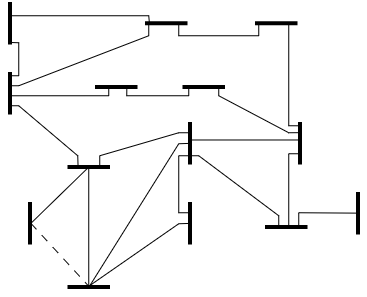}
         \caption{Partially observed IEEE 14-bus system: unobserved links (dashed) are shown for a small cycle.}
         \label{fig:remove_small}
     \end{subfigure}

     % Row 2: C
     \begin{subfigure}[b]{0.5\textwidth}
         \centering
         \includegraphics[width=0.6\textwidth]{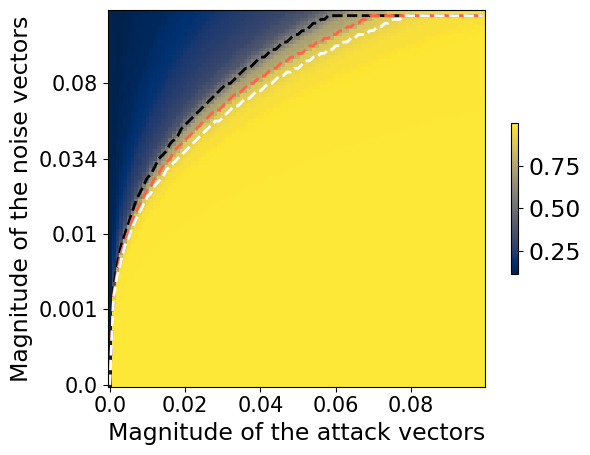}
         \caption{$F_1$ scores (colored heatmap) in the attack-magnitude vs. noise plane for the IEEE 14-bus system. \ac{csd} performance curves corresponding to $F_1$=0.8 are shown for comparison: full observability baseline (red dashed), missed small cycle (white dashed), and missed large cycle (black dashed).}
         \label{fig:remove_f1}
     \end{subfigure}
     
     \caption{Partial observability analysis on IEEE-14 bus system}
     \label{fig:total_figure}
\end{figure}

\begin{enumerate}
    \item Localized in Missing Cycles: If the attack is applied exclusively within the broken (missing) cycles, it becomes mathematically undetectable.
    \item Localized in Remaining Cycles: If the attack is confined to the observable cycles, detection remains feasible.

\item Uniformly Distributed Attack: When an attack is applied uniformly across all links, performance may either improve or degrade depending on the specific cycles lost as follows: 

    \begin{itemize}
    \item Loss of Large Cycles: Performance may counterintuitively improve, as the estimation avoids the high generalization errors associated with longer cycles.
    \item Loss of Small Cycles: Performance may degrade, potentially substantially, if some of the most accurate components of the null space estimation are not captured.
\end{itemize}
\end{enumerate}

While scenarios 1) and 2) are deterministic, scenario 3) demonstrates that the quality of detection under partial observability may be influenced by the topological properties of the remaining sub-graph.

The impact of partial observability on detection performance is illustrated in Figs.~\ref{fig:remove_large} and \ref{fig:remove_small}, which depict missed links (dashed) pertaining to large and small cycles, respectively.
As shown in Fig.~\ref{fig:remove_f1}, the resulting \ac{csd} performance curves for missed small cycle (white dashed) and missed large cycle (black dashed) scenarios are evaluated against a full observability baseline (red dashed).
As expected, the unobservability of a small cycle (white dashed) may lead to a noticeable reduction in detector efficacy, with the performance curve falling below the baseline due to the loss of the high-fidelity information inherent in shorter cycles. In contrast, breaking a large cycle (black dashed) counterintuitively improves performance above the baseline, suggesting that excluding the higher generalization errors associated with longer cycles can refine the overall accuracy of the null space estimation.

\end{document}